\documentclass[10pt,twocolumn,letterpaper]{article}

\usepackage{cvpr}              

\usepackage{graphicx}
\usepackage{amsmath}
\usepackage{amssymb}
\usepackage{booktabs}
\usepackage{style}
\usepackage{mathtools}
\usepackage{cuted}
\usepackage{algorithm}
\usepackage{algpseudocode}
\usepackage{subcaption}
\usepackage{multicol}
\usepackage[accsupp]{axessibility}  

\def \Re {\mathbb{R}}
\def \Diag {\mathrm{Diag}}

\newcommand{\myparagraph}[1]{\medskip \noindent \textbf{#1}}
 
\newenvironment{proof}[1][Proof]{\textbf{#1.} }{\ \rule{0.5em}{0.5em}}

\usepackage[pagebackref,breaklinks,colorlinks]{hyperref}

\usepackage[capitalize]{cleveref}
\crefname{section}{Sec.}{Secs.}
\Crefname{section}{Section}{Sections}
\Crefname{table}{Table}{Tables}
\crefname{table}{Tab.}{Tabs.}

\def\cvprPaperID{11189} 
\def\confName{CVPR}
\def\confYear{2022}

\def\mcr2{MCR$^2$}

\begin{document}

\title{Efficient Maximal Coding Rate Reduction by Variational Forms}

\author{Christina Baek$^*$\\
Carnegie Mellon University\\
{{\tt\small kbaek@cs.cmu.edu}   \ \ {\small *equal contribution}}\\
\and
Ziyang Wu$^*$\\
International Digital Economy Academy\\
{{\tt\small wuziyang@idea.edu.cn}  \ \ {\small *equal contribution}}
\and
Kwan Ho Ryan Chan \\
Johns Hopkins University\\
{\tt \small kchan49@jhu.edu}
\and 
Tianjiao Ding \\
Johns Hopkins University\\
{\tt \small tding@jhu.edu}
\and
Yi Ma \\
University of California, Berkeley\\
{\tt\small yima@eecs.berkeley.edu}
\and
Benjamin D. Haeffele\\
Johns Hopkins University\\
{\tt\small bhaeffele@jhu.edu}
}
\maketitle

\begin{abstract}
The principle of Maximal Coding Rate Reduction (\mcr2) has recently been proposed as a training objective for learning discriminative low-dimensional structures intrinsic to high-dimensional data to allow for more robust training than standard approaches, such as cross-entropy minimization. 
However, despite the advantages that have been shown for \mcr2 training, \mcr2 suffers from a significant computational cost due to the need to evaluate and differentiate a significant number of log-determinant terms that grows linearly with the number of classes.
By taking advantage of variational forms of spectral functions of a matrix, we reformulate the MCR$^2$ objective to a form that can scale significantly without compromising training accuracy. Experiments in image classification  
demonstrate that our proposed formulation results in a significant speed up over optimizing the original MCR$^2$ objective directly and often results in higher quality learned representations.  Further, our approach may be of independent interest in other models that require computation of log-determinant forms, such as in system identification or normalizing flow models. 
\end{abstract}
\section{Introduction}
\label{sec:intro}
Given a classification task, deep networks aim to learn a nonlinear mapping, consisting of a series of linear and nonlinear functions, that can map data to their correct labels. The overall deep network can often be interpreted as a composition of a nonlinear ``featurizer'' $f_\theta$ and a linear classifier $g(\mb z) = \mb W \mb z $ for some matrix $\mb W$. The hidden layers or the \textit{featurizer} is designated with learning a latent representation $\mb z_\theta = f_\theta(\mb x) \in \mathbb{R}^d$ that best facilitates the final layer or \textit{classifier} for the downstream task. 


The canonical way to train a deep learning model for a classification task is empirical risk minimization using cross-entropy (CE) loss. While CE measures the difference between the model's prediction and the true labels, it does not explicitly enforce any structure over the representation. In fact, Papyan, Han, and Donoho \cite{collapse, hui2022limitations} show that this direct label fitting implicitly leads to \textit{neural collapse} in deep networks. That is, as CE loss converges to $0$, the representations of each class at the last hidden layer collapse to a single point, suppressing within-class variability.

Beyond neural collapse and failing to represent within-class variation, several works \cite{shah2020simplicity, geiros2020, scimeca2021} have empirically shown that training neural networks using stochastic gradient descent (SGD) on CE loss often leads the network to utilize the simplest, often spurious, feature in the image for classification. 
This hypothesis is theoretically supported by 
\cite{allenzhu2021distillation} which verified that when multiple explanations can describe a class, models trained with CE often pick a subset of features that can classify a majority of the points well and then classifies the remaining points from noise in the data.

To alleviate this issue, Yu et al. \cite{mcr} proposed a framework for learning geometrically meaningful representations, via a featurizer $f_\theta$, by maximizing the coding rate reduction (MCR$^2$). In brief, the MCR$^2$ objective encourages the latent representation of the entire training set to \textit{expand} or occupy as much volume as possible, while simultaneously pushing each class to \textit{compress} or occupy as little space as possible. Empirically and theoretically, it is shown that this objective drives the latent representations of each class to lie on a low dimensional \textit{linear} subspace, with the subspaces orthogonal to each other, 
which empirically 
provides robustness against label noise, a notable advantage of \mcr2 compared to CE \cite{mcr}.

However, despite these inherent advantages, the MCR$^2$ metric can be costly from a computational perspective. In particular, the loss involves calculating the $\log \det$ of the Gram matrix of the representations of each class. Not only does the number of $\log\det$ terms  grow linearly with the number of classes, but computing (and back-propagating) the $\log \det$ of a $d \times d$ matrix incurs a computational cost of $O(d^3)$. For this reason, \mcr2 methods to date have been limited to datasets with a relatively small number of classes such as MNIST and CIFAR-10, where the loss is computationally feasible.  
In order to make  MCR$^2$ scalable,  there is a significant need to improve the computational efficiency, particularly as it pertains to computing $\log \det$ terms, to allow for large numbers (hundreds or thousands) of classes in high dimensional spaces. 
%
%

\myparagraph{Contributions.} In this paper we make significant progress towards this goal, with the following contributions:
\begin{enumerate}
    \item We provide an alternative formulation of the MCR$^2$ objective based on a variational form of the $\log \det$ function which scales much more gracefully with the number of classes and the problem dimension.
    \item We show experimentally that the alternative formulation requires approximately the same number of epochs to converge as the original \mcr2 formulation, but achieves a significant speedup in the training cost per epoch, particularly as the number of classes in the dataset grows.
    \item As an additional benefit of our formulation, we also observe empirically that training over our proposed variational formulation often results in higher quality learned latent representations and better test accuracy than the original \mcr2 objective.
\end{enumerate}

Finally, we note that our approach for optimization with variational forms may be of independent interest for other models which require computing $\log \det$ terms, such as in system identification \cite{log-det} or {normalizing flow} models \cite{kobyzev2020normalizing}.

\section{Preliminaries}

Here we first describe the original \mcr2 formulation as well as introduce relevant background material.

\subsection{MCR$^2$ Objective}
The original MCR$^2$ objective \cite{mcr} takes the following form: Given $m$ training samples $\mb{X} = [\mb{X}_1, \ldots, \mb{X}_m] \in \R^{D \times m}$ belonging to $k$ classes\footnote{Here we adopt the notation that an upper case letter $\mb X$ represents a matrix and $\mb X_i$ denotes the $i^\text{th}$ column of a matrix.}, let $\mb Z_\theta = [f_\theta(\mb{X}_1),...,f_\theta(\mb{X}_m)] \in \R^{d \times m}$ be the latent representation where recall $f_\theta$ is the featurizer parameterized by $\theta$, and let $\mb \Pi \in \Re^{m \times k}$ define the class membership, where $\mb \Pi_{i,j}$ denotes the probability\footnote{Note that if the labels are known exactly then the entries of $\mb \Pi$ are binary $\{0,1\}$ with each row of $\mb \Pi$ summing to one. Notice that our notation of $\mb \Pi$ is slightly different from that adopted in \cite{mcr}. Our choice is more compact for optimization purposes.} that $\mb X_i$ is in class $j$. 
Then, $\text{MCR}^2$ aims to learn a feature representation $\mb Z_\theta$ that maximizes the following \emph{coding rate reduction} $\Delta R(\mb Z_\theta)$:
\begin{equation}
{\small
\label{eq:mcr}
\begin{aligned}
\max_\theta  \Delta  R(\mb Z_\theta) &\equiv  R(\mb Z_\theta) -  R_{c}(\mb Z_\theta, \mb \Pi) \quad  \text{s.t. } \mb Z_\theta \in \mc S, \quad \text{where} \\
R(\mb Z_\theta) &= \frac{1}{2}\log\det\left(\mb I + \alpha \mb Z_\theta \mb Z_\theta^\top\right), \quad \text{and} \\
R_{c}(\mb Z_\theta, \mb \Pi) &= \sum_{j=1}^k \frac{\gamma_j}{2}\log\det\left(\mb I + \alpha_j \mb Z_\theta \Diag(\mb \Pi_j) \mb Z_\theta^\top\right) 
\end{aligned}
}
\end{equation}
where $\mb \Pi_j$ denotes the $j^\text{th}$ column of $\mb \Pi$, $\Diag(\mb \Pi_j)$ denotes a diagonal matrix with $\mb \Pi_j$ along the diagonal, $\alpha = d/(m\epsilon^2)$, $\alpha_j = d/(\langle \mb 1, \mb \Pi_j \rangle \epsilon^2)$, $\gamma_j = \langle \mb 1, \mb \Pi_j \rangle /m$, $\mc S$ is the set of all matrices whose columns all have unit $\ell_2$ norm\footnote{Note the constraint that $\mb Z_\theta$ has unit norm columns is often achieved by simply having the final operation of the network $f_\theta$ be a normalization.} and $\epsilon > 0$ is a prescribed precision error. Roughly speaking, $R(\mb Z_\theta)$, known as the \textit{expansion} term, captures the dimension (or the volume) of the space spanned by $\mb Z_\theta$ while $R_c(\mb Z_\theta, \mb \Pi)$, or the  \textit{compression} term, measures the sum of the dimensions/volumes of the data from each class. From an information-theoretic point of view, $R(\mb Z_\theta)$ estimates the coding rate, or the number of binary bits required to encode $\mb Z_\theta$, through $\epsilon$-ball packing \cite{ma2007segmentation}.  The terms are called expansion and compression terms respectively, since by maximizing $\Delta R$, the first coding rate term is maximized, which seeks to expand the overall volume of the embedded features, while the second coding rate term is minimized, which seeks to compress the volume of the embedded features from each class. 

By assessing the MCR$^2$ objective \eqref{eq:mcr}, one can  already observe a potential drawback of \mcr2 for optimization. In particular, note that each $\log \det$ term requires $O (\min\{d^3, m^3\})$ operations to compute (and similarly to back-propagate through). While $d \ll D$ can often be made reasonably small for many high-dimensional datasets which have an underlying low-dimensional structure,  $R_c(\mb Z_\theta, \mb \Pi)$ in particular is still often expensive to compute because it requires $k$ computations of $\log \det$. This severely limits the application of \mcr2 for datasets with even moderate numbers, say hundreds, of classes as the objective becomes computationally infeasible on common machines.

\subsection{Variational Forms of Spectral Functions}

To avoid this computational bottleneck, here we propose instead a formulation which takes advantage of variational forms of spectral functions of a matrix.
Specifically, for a given positive semi-definite (PSD)  matrix $\mb M$ and any scalar $c \geq 0$, note that 
\begin{equation}
\label{logprop}
    \log \det(\mb I + c \mb M) = \sum_{i=1}^r \log(1 + c\sigma_i (\mb M)),
\end{equation} where $r$ denotes the rank of $\mb M$ and $\sigma_i (\mb M)$ is the $i^{th}$ singular value of $\mb M$. Here note that $\log(1 + c\sigma)$ is a non-decreasing, concave function of $\sigma$, so we can exploit known variational forms of spectral functions\cite{giampouras2020novel, ornhag2020bilinear}. In particular note the following result in \cite{ornhag2020bilinear}. 
\begin{theorem}[Adapted from \cite{ornhag2020bilinear}]
\label{var-theo}
For any matrix $\mb X$, let r denote the rank of $\mb X$, let $\sigma_i(X)$ denote the $i^{th}$ singular value of $\mb X$, and define $$ H(\mb X) = \sum_{i=1}^r h(\sigma_i(\mb X)).$$ for some function $h$. If $h$ is a concave, non-decreasing function on $[0, \infty)$ with $h(0) = 0$, then the following holds

$$ H(\mb X) = \min_{\mb U, \mb V : \mb U \mb V^\top = \mb X} \sum_{i} h\big(\norm{\mb U_i}{2} \norm{\mb V_i}{2}\big),$$ where $(\mb U_i, \mb V_i)$ denotes the $i^{th}$ columns of $(\mb U, \mb V)$.  Note also that $(\mb U, \mb V)$ can have an arbitrary number of columns $(\geq r)$ provided $\mb U \mb V^\top = \mb X$.
\end{theorem}

\section{Proposed Formulation}

Having introduced the above background material, we now describe our proposed approach.  In particular, note that Theorem \ref{var-theo} immediately gives the following result as a proposition.
\begin{proposition}
\label{var-cor}
Let $\mb M $ be any real positive semi-definite matrix and let $c \geq 0$ be any non-negative scalar. Then the following holds:
\begin{equation}
\label{var-eq}
{
-\log \det(\mb I + c \mb M) = 
\max_{\substack{\mb U : \ \mb U \mb U^\top = \mb M}} \!\! - \!\! \sum_{i} \log\big(1 + c\norm{\mb U_i}{2}^2\big).
}
\end{equation}
Further, if $\bar{\mb U} \mb S \bar{\mb U}^\top = \mb M$ is a SVD of $\mb M$ then $\mb U^* = \bar{\mb U}\mb S^{1/2}$ is a solution to the above problem.
\end{proposition}
\begin{proof} First recall the basic fact that for any function $\psi(x)$ one has $-\min_x \psi(x) = \max_x -\psi(x)$. Additionally, recall \eqref{logprop} and note that the function $h(x) = \log(1 + c x)$ satisfies the conditions required for $h$ in Theorem \ref{var-theo}.  These facts give the following:
\begin{equation}
\label{eq:neg_var_uv}
{
\begin{aligned}
-&\log \det(\mb I + c \mb M) = \\ 
&\max_{\substack{\mb U, \mb V : \mb U \mb V^\top = \mb M}} \!\! - \!\! \sum_{i} \log(1 + c\norm{\mb U_i}{2}\norm{\mb V_i}{2}).
\end{aligned}
}
\end{equation}
Further, note \eqref{eq:neg_var_uv} implies that $-\log \det (\mb I + c \mb M) \geq - \sum_i \log (1 + c \norm{ \mb U_i}{2}^2)$ for all $\mb U$ such that $\mb U \mb U^\top = \mb M$ since we have simply added the constraint $\mb U = \mb V$.  The result is completed by noting that for the choice of $\mb U = \mb U^*$ the maximum can be attained since $\log(1 + c \norm{\mb U^*_i}{2}^2) = \log(1+ c \sigma_i(\mb M)), \forall i \in [r]$.
 \end{proof}

\subsection{Variational Formulation of MCR$^2$}
Using Proposition \ref{var-cor}, we develop our formulation by replacing the $-\log \det$ terms in $R_c$ of \eqref{eq:mcr} by the above variational form.  In particular, note that for each class $j$ we can eliminate the associated $-\log \det$ term in $R_c$ of \eqref{eq:mcr} by introducing an additional matrix $\mb U^{(j)}$ subject to the constraint that $\mb Z_\theta \mathrm{Diag}(\mb \Pi_j) \mb Z_\theta^\top = \mb U^{(j)} (\mb U^{(j)})^\top$.  Further, due to the fact that each row of $\mb \Pi$ sums to one, we also have $\sum_{j=1}^k \mb Z_\theta \Diag(\mb \Pi_j) \mb Z_\theta^T = \mb Z_\theta \mb Z_\theta^\top$.  As such, the variational form in Proposition \ref{var-cor} gives that the original \mcr2 objective in \eqref{eq:mcr} is equivalent to the following constrained variational form:
%
%
\begin{equation}
\label{eq:U_var}
{
\begin{aligned}
    &\max_\theta \Delta R(\mb Z_\theta) = \\ 
    &\max_{\theta, \{\mb U^{(j)}\}_{j=1}^k} \frac{1}{2} \log \det \left(\mb I +\alpha \sum_{j=1}^k \mb U^{(j)}(\mb U^{(j)})^\top \right)  \\ 
    &\quad \quad \quad \quad - \sum_{j=1}^k \frac{\gamma_j}{2} \sum_i \log \left(1 + \alpha_j \norm{\mb U_i^{(j)}}{2}^2\right) \\
    & \quad \text{s.t. }\forall j, \mb U^{(j)} (\mb U^{(j)})^\top = \mb Z_\theta \textrm{Diag}(\mb \Pi_j) \mb Z_\theta^\top \ \text{and} \ \mb Z_\theta \in \mc S.
\end{aligned}
}
\end{equation}
From this form, we now reparameterize the $\mb U^{(j)}$ matrices as $\mb U^{(j)} = \Gamma \Diag(\mb A_j)^{1/2}$  where $\Gamma \in \Re^{d \times q} \cap \mc S$ is a dictionary with unit norm columns, and $\mb A_j \in \Re_+^q$ is a (non-negative) encoding vector.  Now let $\mb A \in \Re_+^{q \times k} = [\mb A_1, \ldots , \mb A_k]$ be a matrix of the concatenated encoding vectors and note that we trivially have $\Gamma \Diag(\mb A_j) \Gamma^\top = \mb U^{(j)} (\mb U^{j})^\top$ and $\| \mb U^{(j)}_i \|_2^2 = \mb A_{i,j}$, which gives another equivalent formulation for the \mcr2 objective, provided the number of dictionary elements $q$ is sufficiently large so that each optimal $\mb U^{(j)}$ matrix in \eqref{eq:U_var} can be encoded by $\Gamma$ (i.e., each column of $\mb U^{(j)}$ must be a column of $\Gamma$ within a scaling factor):
\begin{equation}
{
\begin{aligned}
    &\max_\theta \Delta R(\mb Z_\theta) = \\ 
 & \max_{\theta, \Gamma \in \R^{d \times q} \cap \mc S,  \mb A \in \R_+^{q \times k}} \frac{1}{2} \log \det \left(\mb I + \alpha \sum_{j=1}^k \Gamma \mathrm{Diag}(\mb A_j) \Gamma^\top \right) \\  
  &\quad \quad \quad - \sum_{j=1}^k \frac{\gamma_j}{2} \sum_{l=1}^q \log \left(1 + \alpha_j \mb A_{l, j}\right) \\
 &\text{s.t. } \forall j,  \Gamma \Diag(\mb A_j) \Gamma^\top = \mb Z_\theta \textrm{Diag}(\mb \Pi_j) \mb Z_\theta^\top\ \text{and} \ \mb Z_\theta \in \mc S.
\end{aligned}
}
\label{eqn:variational-constrained}
\end{equation}
Finally, we relax the strict equality constraints $\Gamma \Diag(\mb A_j) \Gamma^\top = \mb Z_\theta \textrm{Diag}(\mb \Pi_j) \mb Z_\theta^\top$ with $\ell_2$ penalties $\frac{1}{\gamma_j} \|\mb Z_\theta \mathrm{Diag}(\mb \Pi_j) \mb Z_\theta^\top - \Gamma \mathrm{Diag}(\mb A_j) \Gamma^\top\|_F^2$ to arrive at our final proposed formulation, which we call V-MCR$^2$:
%
%
\begin{equation}
{\small
\begin{aligned}
 & \max_{\theta, \Gamma \in \R^{d \times q} \cap \mc S,  \mb A \in \R_+^{q \times k}} R^v(\Gamma, \mb A) - R^v_c(\mb A) - \frac{\mu}{2m} M(\mb Z_\theta, \Gamma, \mb A) \\
  &\text{ where} \quad  R^v(\Gamma, \mb A) = \frac{1}{2} \log \det \left(\mb I + \alpha \sum_{j=1}^k \Gamma \mathrm{Diag}(\mb A_j) \Gamma^\top \right), \\  
 & R^v_c(\mb A) = \sum_{j=1}^k \frac{\gamma_j}{2} \sum_{l=1}^q \log \left(1 + \alpha_j \mb A_{l, j}\right), \\
 & M(\mb Z_\theta, \Gamma, \mb A) = \sum_{j=1}^k \frac{1}{\gamma_j} \norm{\mb Z_\theta \mathrm{Diag}(\mb \Pi_j) \mb Z_\theta^\top \!-\! \Gamma \mathrm{Diag}(\mb A_j) \Gamma^\top}{F}^2,  \\ 
\end{aligned}
}
\label{eqn:variational-objective}
\end{equation}
such that $\mb Z_\theta \in \mc S$. Regularization parameter $\mu > 0$ weights how strictly the equality constraints should be approximated, and the $\frac{1}{\gamma_j}$ terms roughly ensure class balance (recall, $\gamma_j = \langle \mb 1, \mb \Pi_j \rangle /m$).  From this reformulation, we have significantly reduced the complexity of evaluating the objective function. The $\log \det$ terms in $R_c$ which take $O(k \min \{d^3, m^3\})$ time to evaluate is now replaced by $O(qk)$ for the sum of the $\log(1 + \alpha_j \mb A_{l,j})$ terms along with the cost of computing the $M$ term which scales as $O(kd^2)$.  

\subsection{Interpretation of the Variational Form}
Besides the above computational advantages, we also discuss a few additional aspects of our formulation below.

\myparagraph{Sparsifying dictionary learning interpretation.} Notice that the above variational reformulation takes on a natural interpretation as learning a sparsifying dictionary: it essentially ``parameterizes'' the subspaces spanned by each class with a common shared dictionary $\Gamma$. Every class then selects a ``sparse'' number of eigenbases, with $\mb A_j$,  from this dictionary and forms its estimate of the sample covariance within the subspace. Notice that the scalar $\log$ terms in $ R^v_c(\mb A)$ are precisely nonconvex sparsity promoting measures adopted in early studies of sparse representation \cite{olshausen1996emergence,olshausen1997sparse}. The sparsity in the (spectral) bases in terms of $\mb A_j$ precisely corresponds to the subspace spanned by each class being low-dimensional or low-rank.  

\myparagraph{Penalty function method and other options.} Notice that in our formulation \eqref{eqn:variational-objective}, the equality constraint in \eqref{eqn:variational-constrained} is enforced through a penalty function $M$. As the penalty weight $\mu$ increases to infinity and the dictionary $\Gamma$ is sufficiently large\footnote{In the worst case the model become equivalent when $\Gamma$ is large enough to contain a concatenation of the singular vectors of each class $\mb Z_\theta \Diag (\mb \Pi_j) \mb Z_\theta^\top$.}, the formulation becomes exactly equivalent to the original formulations \eqref{eqn:variational-constrained} and \eqref{eq:mcr}. Of course, to deal with the equality constraint in \eqref{eqn:variational-constrained} more precisely, one may also consider adopting more advanced methods such as the augmented Lagrangian multiplier method to incorporate the equality constraint \cite{hestenes1969multiplier,rockafellar1973multiplier}, which we leave for future work. However, as we discuss next, by relaxing the strict equality constraint we also gain a potential advantage when the latent representation contains  noise.  

\myparagraph{Low-rank LASSO interpretation.} Notice that the sparse/low-rank promoting term $ R^v_c(\mb A)$ and the quadratic penalty term $M$ together resemble the classic LASSO method  for recovering a sparse solution from noisy measurements \cite{TibshiraniR1996}. The only difference here is that we are seeking a sparse solution in the spectrum of a covariance matrix -- hence seeking a low-rank solution for the covariance. So to some extent, one may consider the variational form as a ``low-rank LASSO.'' We have noticed a nice side benefit of this LASSO-type formulation: empirically it seems to lead to better solutions than solving the original MCR$^2$ objective (see experimental results in Section \ref{sec:experiment-results}). Part of the reason is likely because the LASSO type relaxation introduced by the variational form finds a solution that is more stable to small noise in the data or deviation from an ideal low-dimensional linear subspace.




\subsection{Optimization Strategy}
\subsubsection{Alternating Maximization}
To optimize \eqref{eqn:variational-objective}, we adopt an alternating maximization strategy \cite{Eckstein2012} between the variation parameters $(\Gamma, \mb A)$ and the network parameters $(\theta)$. At each iteration, we first optimize $\Gamma$ and $\mb A$ by taking one step of a proximal gradient ascent update which consists of taking a gradient ascent step on the relevant part of variational loss $\Delta R^v - \frac{\mu}{2m} M(\mb Z_\theta, \Gamma, \mb A)$, followed by normalizing the columns of $\Gamma$ to have unit $\ell_2$ norm and thresholding the negative entries of the updated $\mb A$ matrix to $0$ (i.e., applying the ReLU function). 
To ensure stability of the gradient-based method, we inversely scale the learning rate of $\Gamma$ and $\mb A$ by upper-bounds of the Lipschitz constants of the gradients, $\frac{1}{L_\Gamma}$ and $\frac{1}{L_{\mb A}}$, respectively. (See Appendix for our derivation for the bounds.)  
Next, the matrix approximation term $M(\mb Z_\theta, \Gamma, \mb A)$ is recomputed using the updated $\Gamma$ and $\mb A$, and we then update the network parameters, $\theta$, by taking a gradient step on the relevant part of the variation loss, $\nabla_\theta M(\mb Z_\theta, \Gamma, \mb A)$. 

In addition, note that from the variational form in Proposition \ref{var-cor} we know that the optimal variational parameters $(\Gamma, \mb A)$ should be closely related to the singular values and vectors of $\mb Z_\theta  \Diag(\mb{\Pi}_j) \mb Z_\theta^\top$ for each class $j$ (the relationship becomes exact for large values of $\mu$).  We exploit this fact to initialize the variational parameters and to make periodic `approximately closed-form' updates to the variational parameters.
We call this procedure $\textit{latching}$, which we describe in more detail in the next Section \ref{sec:latching}. 
We summarize our overall training process in Algorithm \ref{alg}.\footnote{Note that for clarity we describe the full procedure for gradient ascent, but in our experiments stochastic gradient ascent is implemented.}

\begin{algorithm}
\caption{Variational MCR$^2$ Training}
\begin{algorithmic}[1]
\State {\textbf{Input:} data $\mb X$, labels $\mb Y$, featurizer $f_\theta(\cdot)$}, 
\texttt{latch-freq}, step sizes ($\nu_\theta$, $\nu_\Gamma$, $\nu_{\mb A}$)
\State Initialize $\mb A, \Gamma \leftarrow \mathrm{latching}(\mb X, \mb Y, f_\theta)$
\For{$\texttt{iter} = 0, 1, ..., n-1$}
\State Get $\mb Z_\theta = f_\theta(\mb X)$ and membership matrices $\mb \Pi$
\State Get $\ell_{\text{V-MCR}^2}(\mb Z_\theta, \Gamma, \mb A) =$ \eqref{eqn:variational-objective} 
\State Compute $L_{\mb A}, L_\Gamma$ (see Appendix)
\State $\Gamma \leftarrow \Gamma + \frac{\nu_{\Gamma}}{L_\Gamma} \nabla_\Gamma \ell_{\text{V-MCR}^2}(\mb Z_\theta, \Gamma, \mb A)$
\vspace{1.5pt}
\State $\mb A \leftarrow \mb A + \frac{\nu_{\mb A}}{L_{\mb A}} \nabla_{\mb A} \ell_{\text{V-MCR}^2}(\mb Z_\theta, \Gamma, \mb A)$
\State Project $\mb A \leftarrow \mathrm{ReLU}(\mb A)$
\State Project $\Gamma_l \leftarrow \frac{1}{\norm{\Gamma_l}{2}}\Gamma_l \quad \forall l \in [q]$
\State Recompute $M(\mb Z_\theta, \Gamma, \mb A)$
\State $\theta \leftarrow \theta - \nu_\theta \nabla_\theta(M(\mb Z_\theta, \Gamma, \mb A))$
\If {$\texttt{iter}\,\mbox{mod}\, \texttt{latch-freq} = 0$}
\State $\mb A, \Gamma \leftarrow  \mathrm{latching}(\mb X, \mb Y, f_\theta)$
\EndIf
\EndFor

\State \Return $f_\theta$
\end{algorithmic}\label{alg}
\end{algorithm}

\begin{algorithm}
\caption{Latching}
\begin{algorithmic}
\State {\textbf{Input:} data $\mb X$, labels $\mb Y$, featurizer $f_\theta(\cdot)$}
\State Get $\mb Z_\theta = f_\theta(\mb X) \in \R^{d \times m}$ and membership $\mb \Pi \in \Re^{m \times k}$
\State $\mb A \leftarrow \mb 0 \in  \R^{q \times k} $ (assume $q$ is divisible by $k$)
\State $\Gamma  \leftarrow \mb 0 \in \R^{d \times q}$ 
\For {$j = 1, ..., k$}
\State Get $\mb U \mathrm{Diag}(\mb \sigma) \mb V^\top = \mathrm{SVD}(\mb Z_\theta \Diag(\mb \Pi_{j}) \mb Z_\theta^\top)$
\State $s \leftarrow q/k$  
\State $\Gamma[:, (j-1)*s : j*s] = \mb U[:, 0:s]$ \% python indexing
\State $\mb A[(j-1)*s : j*s, j] = \mb \sigma[0:s]$ \% python indexing
\EndFor
\State \Return $\mb A, \Gamma$
%
%
\end{algorithmic}\label{latch-alg}
\end{algorithm}


\subsubsection{Latching} 
\label{sec:latching}
In order to optimize the variational MCR$^2$ objective, the dictionary $\Gamma$ and $\mb A$ must maximize $\Delta R(\mb Z_\theta) = R^v(\Gamma, \mb A) - R^v_c(\mb A)$ whilst minimizing the $\ell_2$ regularization term $M(\mb Z_\theta, \Gamma, \mb A)$. This trade-off is controlled by the regularization constant $\mu$. In practice, when $\mu$ is too large, each gradient step does not allow for $\Gamma\mathrm{Diag}(\mb A_j) \Gamma^\top$ to stray too far away from $\mb Z_\theta \Diag(\mb \Pi_j) \mb Z^\top$, which can result in slow convergence. 
We observe that the following procedure improves convergence in practice. Note that the variational form is maximized\footnote{The maximization is exact as $\mu$ becomes large, but a good approximation otherwise.} when $\Gamma$ and $\mb A$ are derived from the SVDs of $\mb Z_\theta \Diag(\mb \Pi_j) \mb Z_\theta^\top$, as given in Proposition~\ref{var-cor}.  This gives a means to periodically reinitialize the variational parameters $(\Gamma, \mb A)$ based on the SVDs of $\mb Z_\theta \Diag(\mb \Pi_j) \mb Z_\theta^\top$, which we refer to as \textit{latching} as described in detail in 
Algorithm \ref{latch-alg}. 
This latching step can be viewed as taking an (approximate) full-maximization step w.r.t. the variational parameters (as opposed to a proximal gradient descent step) based on the closed-form solution provided in Proposition~\ref{var-cor}. This will be an exact maximization step as $\mu$ becomes large.
In short, given a dictionary with $q$ columns, we initialize the dictionary as the concatenation of the top $q/k$ singular vectors of  $\mb Z_\theta \Diag(\mb \Pi_j) \mb Z_\theta^\top$ for each class $j \in [k]$. Similarly, the columns of $\mb A$ are initialized as the corresponding singular values. 
Though latching is in itself an expensive procedure, requiring one to compute the SVD of a matrix $k$ times, it is optional (though we notice a benefit in practice) and can be done relatively infrequently throughout training with a proper choice of the hyper-parameter \texttt{latch-freq}. As a result, the amortized cost of latching becomes  insignificant.

\section{Experimental Setup}
\label{sec:experiment-setup}
We compare $\Delta R$, wall-clock time, and accuracy of models trained with the original MCR$^2$ objective and the variational MCR$^2$ objective on MNIST \cite{lecun1998gradient}, CIFAR-10 \cite{krizhevsky2009learning}, CIFAR-100 \cite{krizhevsky2009learning}, and Tiny ImageNet \cite{le2015tiny} (with 200 classes) datasets. We also compare the performance to cross-entropy (CE) training as a benchmark of the correctness of the learned representations (for classification). The high-level goal of these experiments is to show that 1) the variational MCR$^2$ objective is feasible for datasets where the original MCR$^2$ objective is computationally expensive (or impossible) to train (such as CIFAR-100 and Tiny ImageNet), 2) show that training on the variational MCR$^2$ maximizes the true $\Delta R$ objective and obtains the desired  subspace-like representations.





\subsection{Hyperparameters}
\label{sec:hyperparam}
For fair comparison across training objectives (original MCR$^2$, variational MCR$^2$), we use a learning rate of $10^{-3}$ for the network optimizer and the same batch size. For CE, we use the same batch size, but a larger learning rate of $10^{-2}$. The batch size is 1000 for MNIST and CIFAR-10, and 2000 for CIFAR-100 and Tiny ImageNet. The network is optimized using stochastic gradient descent for all objectives. For the precision error $\epsilon$ of the MCR$^2$ objectives, we use $\epsilon^2 = 0.5$ for all datasets. $\epsilon$ and batch size are consistent with the experimental settings in the original MCR$^2$ work \cite{mcr} for MNIST and CIFAR-10.  For  Variational MCR$^2$, the regularization constant $\mu=1$ and initial learning rates $\nu_\Gamma = 5$, $\nu_{\mb A} = 5$ across all experiments. We perform latching every 50 epochs for all experiments. See Appendix for precise details. The dictionary size $q$ and feature dimension $d$ varies across the datasets. For MNIST and CIFAR-10 we use $d = 128$ and $q = 20 \cdot k$, and for CIFAR-100 and Tiny ImageNet we use $d = 500$ and $q = 10 \cdot k$. 

\subsection{Nearest Subspace Classifier}
\label{sec:nearsub}
MCR$^2$ is a loss over the featurizer $f_\theta$. To  classify the test data, we use the nearest-subspace classifier similar to the original $\text{MCR}^2$ work \cite{mcr}. As shown by \cite{mcr}, at the global optima of MCR$^2$, the representations of each class lie on low-dimensional subspaces that are orthogonal to each other. Yu, et al. \cite{mcr} also empirically observe this property for networks trained by SGD. Assuming that the learned representations satisfy this property, given a test datapoint, we can simply identify the closest subspace for the final classification. Formally, given a test sample $\mb z_{test} = f_\theta(\mb x_{test})$, the predicted label is given as
\begin{equation}
    y = \argmin_{j \in {1,...,k}}\normsq{(\mb I - \mb V^{(j)} (\mb V^{(j)})^\top)\mb z_{test}}{2}{2},
\end{equation}
\noindent where $\mb V^{(j)}$ is a matrix of the top $\lfloor{\frac{d}{k}}\rfloor$ principal components of $\mb Z_\theta \Diag(\mb \Pi_j) \mb Z_\theta^\top$ with $\mb Z_\theta = f_\theta(\mb X)$ being the embedding of the training data $\mb X$.

\section{Experimental Results}
\label{sec:experiment-results}

We discuss the performance of variational MCR$^2$ below. Performance is measured by 1) training speed, 2) the true $\Delta R$ value the model reaches over the training data, and 3) the classification accuracy over the test set. Note that we are more interested in comparing the accuracy and training efficiency between MCR$^2$ and V-MCR$^2$, over trying to reach state of the art results in this paper.

\begin{figure*}[t]
\begin{multicols}{2}
    \centering
    \begin{subfigure}{0.48\textwidth}
      \centering
      \includegraphics[width=\linewidth]{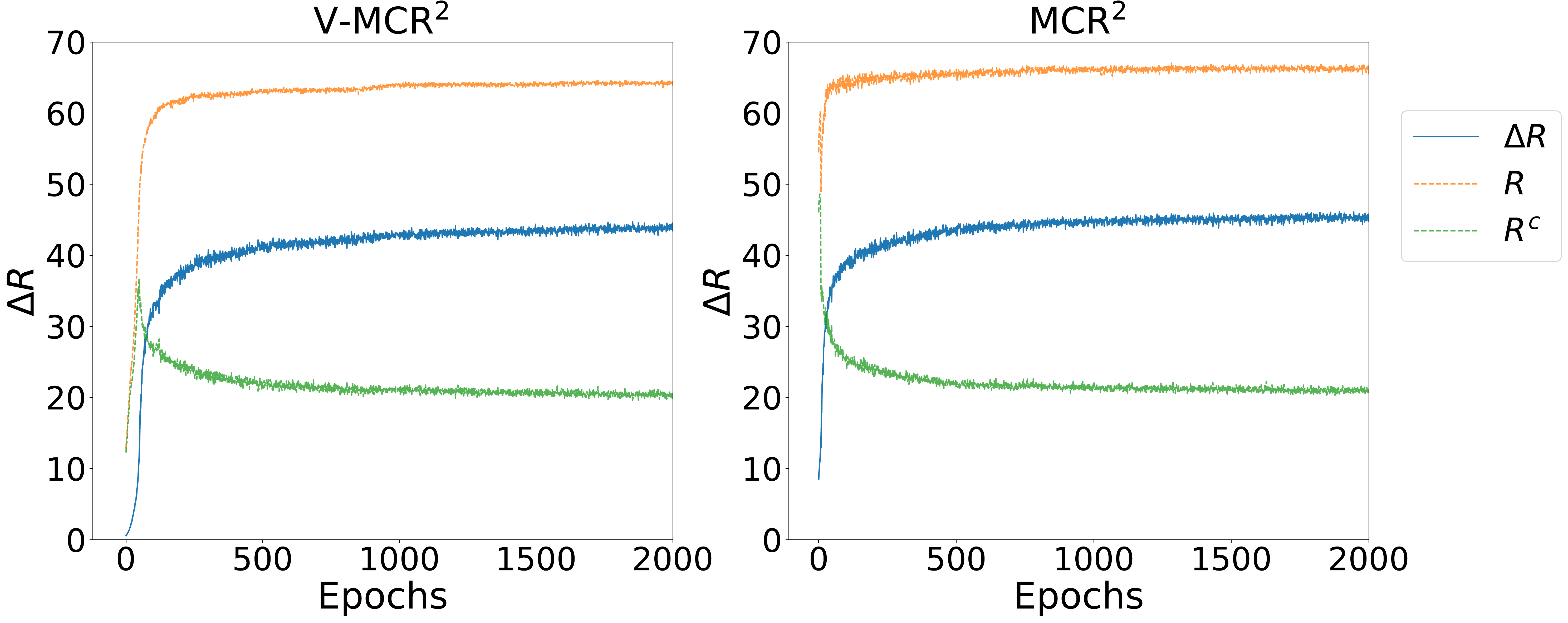}
      \caption{\textbf{MNIST}}
      \label{fig:mnist_loss}
    \end{subfigure} 
    \begin{subfigure}{0.48\textwidth}
      \centering
      \includegraphics[width=\linewidth]{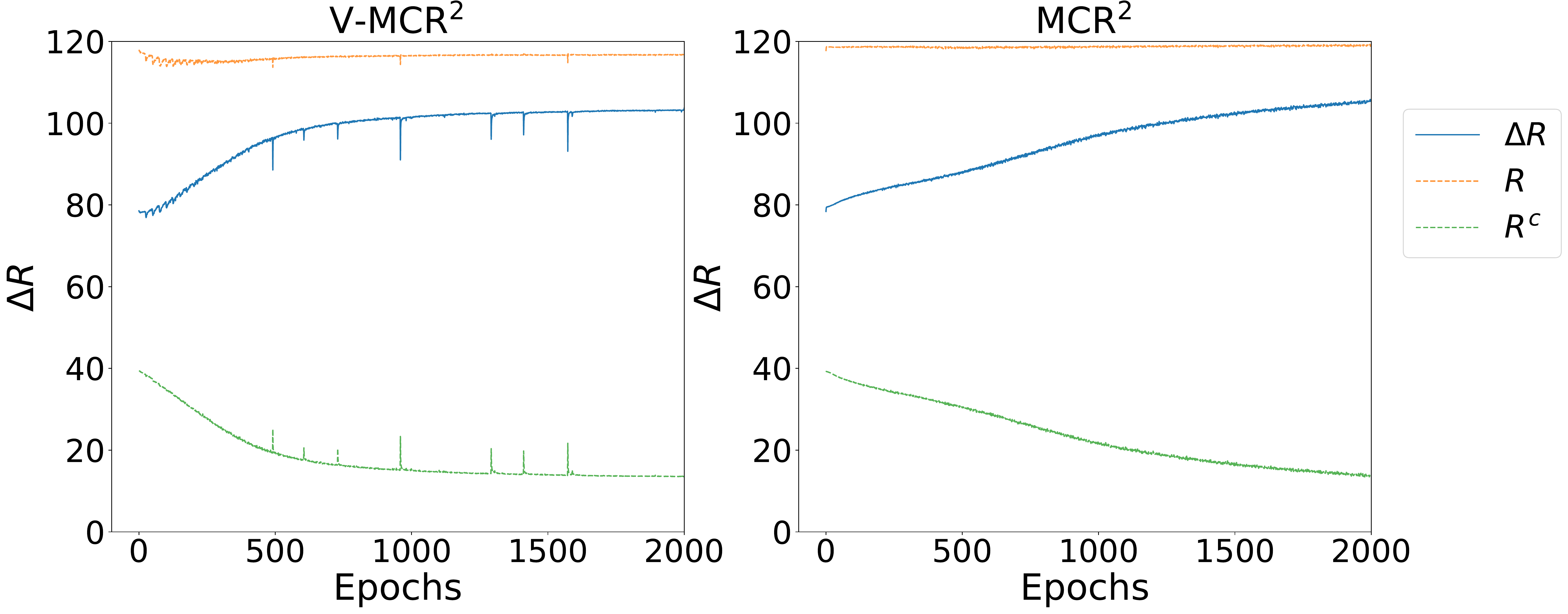}
      \includegraphics[width=\linewidth]{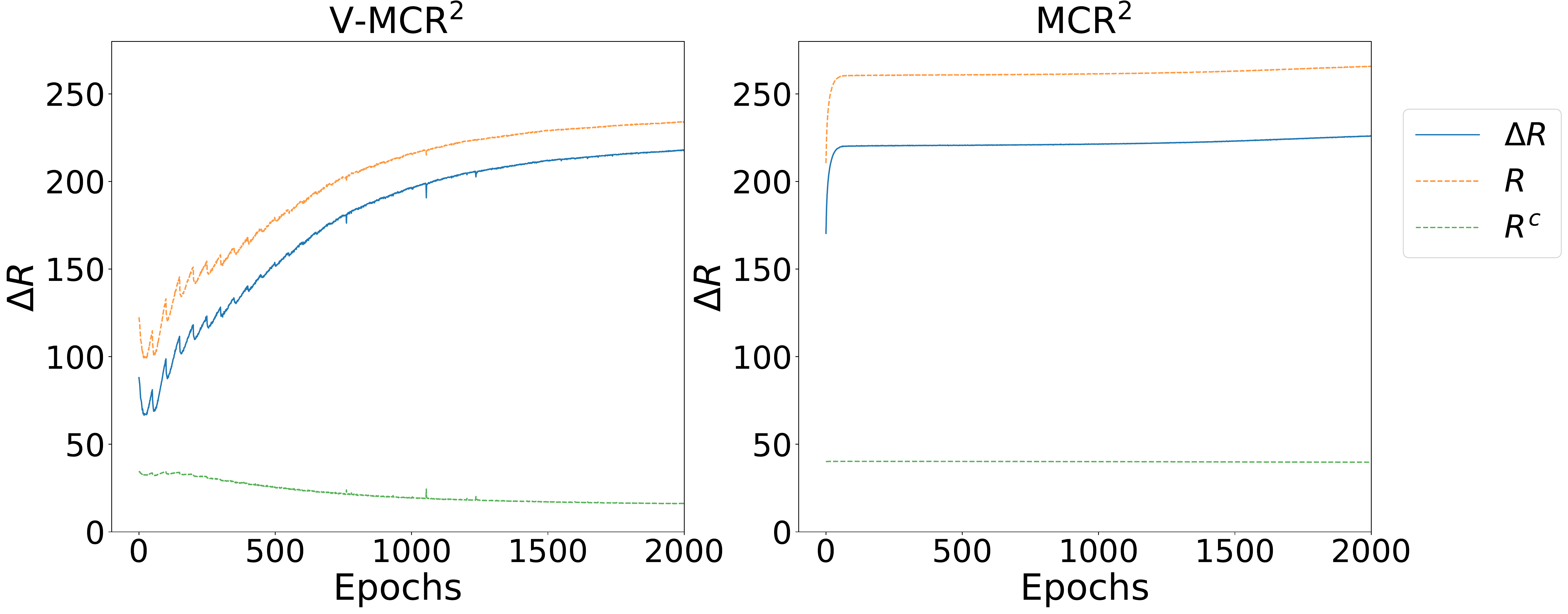}
      \caption{\textbf{CIFAR-100} $d = 100$ (top), $d = 500$ (bottom)}
      \label{fig:cifar100_loss}
    \end{subfigure}
    \begin{subfigure}{0.48\textwidth}
      \centering
      \includegraphics[width=\linewidth]{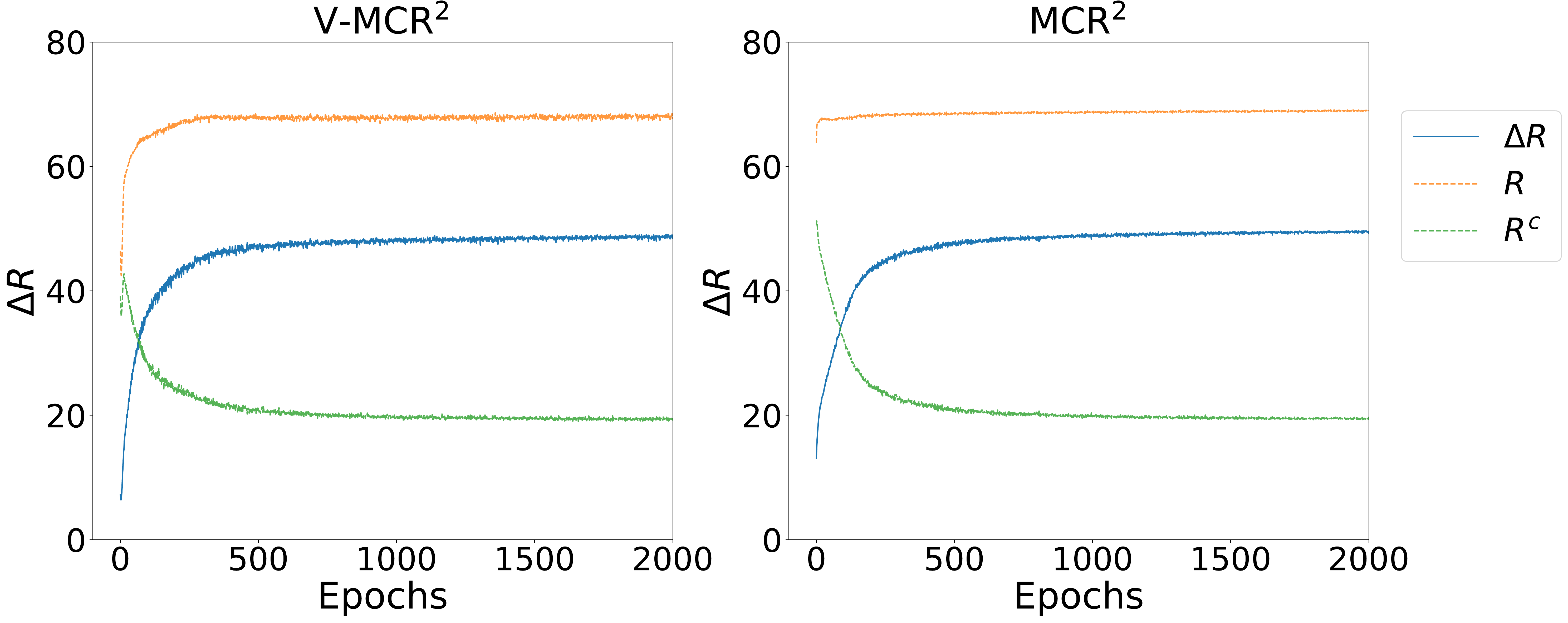}
      \caption{\textbf{CIFAR-10}}
      \label{fig:cifar10_loss}
    \end{subfigure} 
    \begin{subfigure}{0.48\textwidth}
      \centering
      \includegraphics[width=\linewidth]{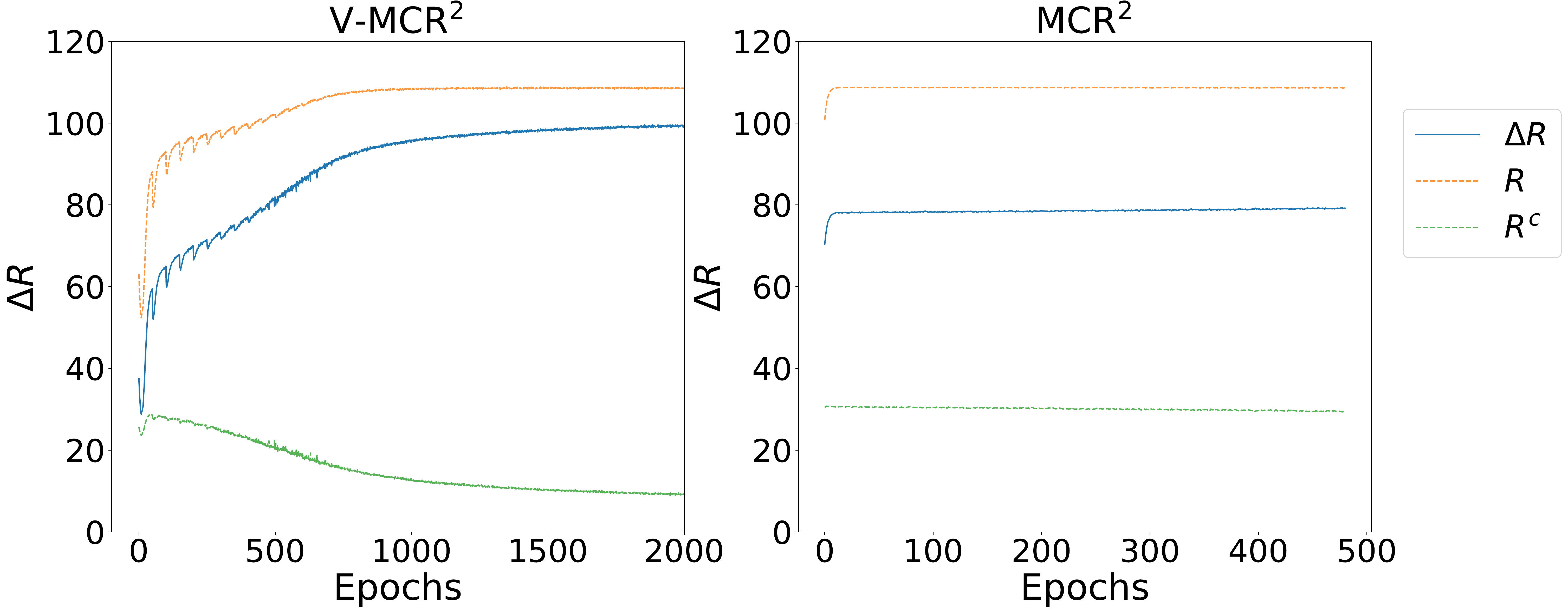}
      \includegraphics[width=\linewidth]{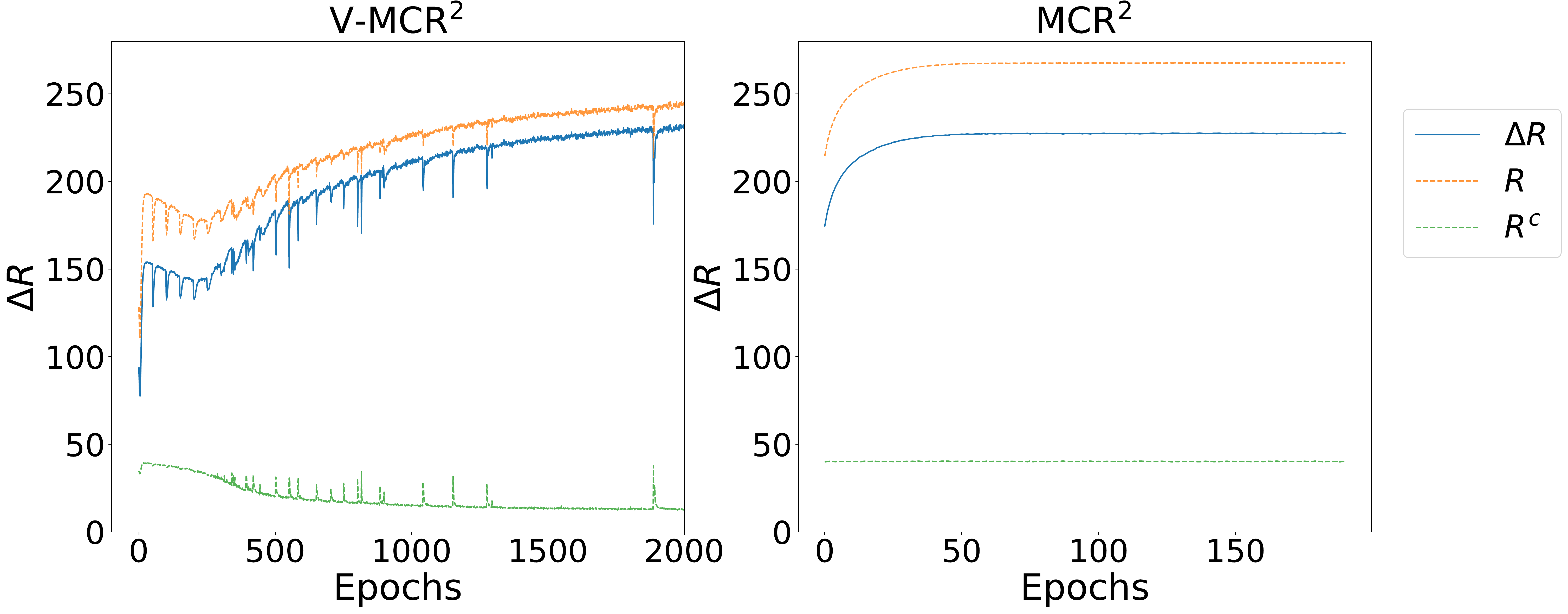}
      \caption{\textbf{Tiny ImageNet} $d = 200$ (top), $d = 500$ (bottom)}
      \label{fig:imagenet_loss}
    \end{subfigure}
    \end{multicols}
    \vspace{-2mm}
    \caption{\textbf{Convergence of training }$\mb{\Delta R}$. We compare the $\Delta  R(  \mb{Z}_\theta)$ of the training data over epochs for V-MCR$^2$ and MCR$^2$. For both V-MCR$^2$ and MCR$^2$, the network is optimized by stochastic gradient descent with a learning rate of $10^{-3}$. All training runs are 2000 epochs, excluding the MCR$^2$ runs for Tiny ImageNet, which we stop at 500 and 200 for $d=200$ and $d=500$ respectively, due to significant computational cost from the high number of classes. Also, note that for V-MCR$^2$, we often observe small undulations in the training loss due to the regular reinitialization by latching.\vspace{-2mm} } 
    \label{fig:convergence}
\end{figure*}

\subsection{Computational Efficiency}
Across datasets, we compare the wall-clock time to train one epoch using the MCR$^2$ \eqref{eq:mcr} and V-MCR$^2$ \eqref{eqn:variational-objective} formulations. The batch size is set to be the same for both models, and all our experiments are performed using PyTorch 1.9.0 and Python 3.8.11 on Nvidia A100-SXM4 GPUs with 40GB of CUDA memory for fair comparison. 
As shown in Table \ref{tab:wall-clock}, V-MCR$^2$ training completes approximately $5\times$ faster on CIFAR-100 and $12\times$ faster on Tiny ImageNet. Even for datasets containing a small number of classes, i.e., MNIST and CIFAR-10, we observe a $1.5-2\times$ speedup. Note that the overhead for the original MCR$^2$ model escalates significantly as the number of classes increases, so we expect even greater improvements in training efficiency with datasets with more classes. 
Even for Tiny ImageNet, training until convergence using MCR$^2$ becomes nearly impractical, while it is easily handled by V-MCR$^2$.

\begin{table}[!htb]
    \centering
    \begin{tabular}{|c|c|c|}
    \hline
    \textbf{Dataset} & \textbf{MCR$^2$} & \textbf{V-MCR$^2$} \\
         \hline 
          MNIST & 11.56  &  6.29\\ 
          \hline 
         CIFAR-10 & 33.06  & 20.71\\ 
          \hline 
         CIFAR-100 & 157.45  &  31.14\\ 
          \hline 
         Tiny ImageNet & 527.85  & 44.23\\ 
          \hline
    \end{tabular}
    \caption{\textbf{Wall-Clock Time per Epoch (in seconds).} We compare the wall-clock time to complete one epoch of training. MNIST and CIFAR-10 use a batch size of 1000. For CIFAR-100 and Tiny ImageNet, we use a batch size of 2000. }
    \label{tab:wall-clock}
\end{table}

Additionally, we compare the true $\Delta R(\mb Z_\theta)$ (i.e., computing the original \mcr2 objective with the iterates of the V-\mcr2 model) over training epochs and observe that both models 1) take approximately the same number of training epochs to converge and 2) reach approximately the same final $\Delta R(\mb Z_\theta)$ objective value at convergence.
Thus, V-MCR$^2$ does not require additional epochs to obtain a good solution which might offset the increased efficiency per epoch. As shown in Figure \ref{fig:convergence}, on MNIST and CIFAR-10, V-MCR$^2$ and MCR$^2$ follow a similar training loss trajectory across epochs. For CIFAR-100, we observe that the convergence rate depends on the dimension of the features/representations, $d$. For a feature dimension of $d=100$ for CIFAR-100, we observed similar number of epochs to convergence for MCR$^2$ and V-MCR$^2$. With a feature dimension of $d=500$, we observe across 5 seeds that training with the original MCR$^2$ leads to a rapid convergence to a poor local optima where the expansion term $R(\mb Z)$ increases rapidly but the compression term $R_c(\mb Z)$ remains the same. On the other hand, although V-MCR$^2$ requires more epochs to converge, we observe the standard expected behavior where $R(\mb Z)$ increases and $R^c(\mb Z)$ decreases and the final solution is of much higher quality (see \ref{sec:representations}). Experiments on Tiny ImageNet show similar behaviors.


\begin{figure}[!htb]
    \centering
    \begin{subfigure}{.45\textwidth}
      \centering
      \includegraphics[width=0.45\linewidth]{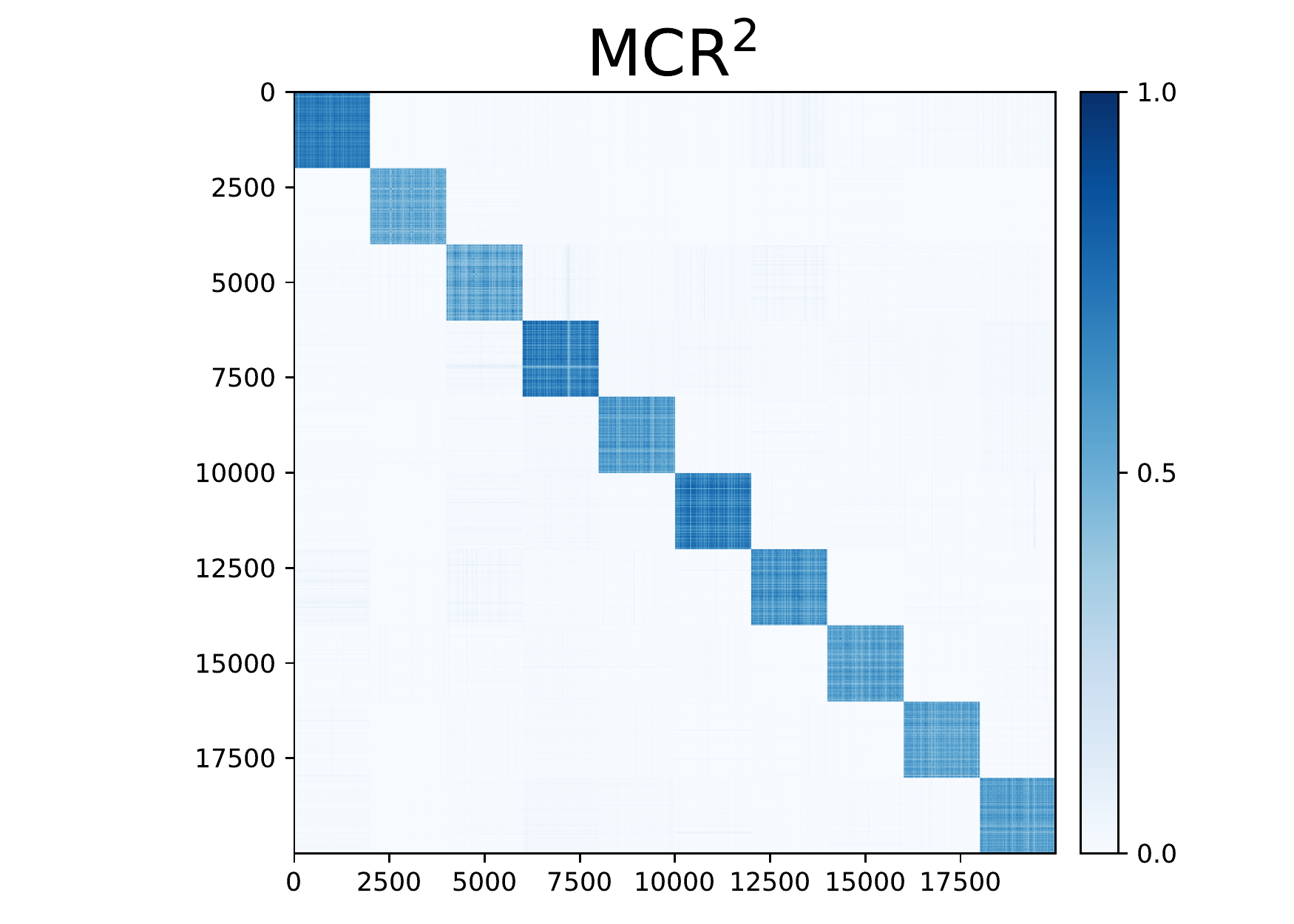}
      \includegraphics[width=0.45\linewidth]{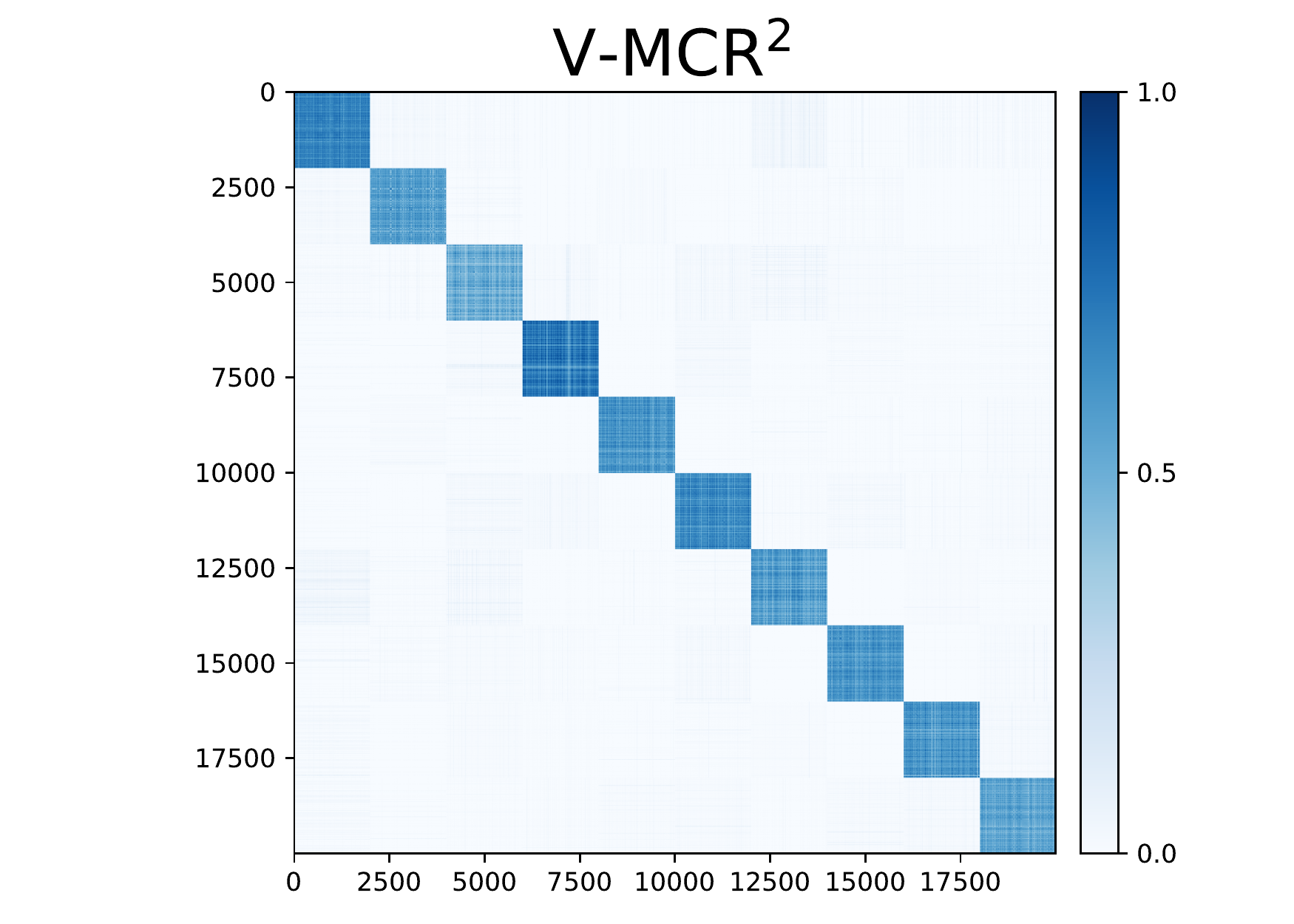}
      \caption{MNIST}
      \label{fig:mnist_heatmap}
    \end{subfigure} 
    \begin{subfigure}{.45\textwidth}
      \centering
      \includegraphics[width=0.45\linewidth]{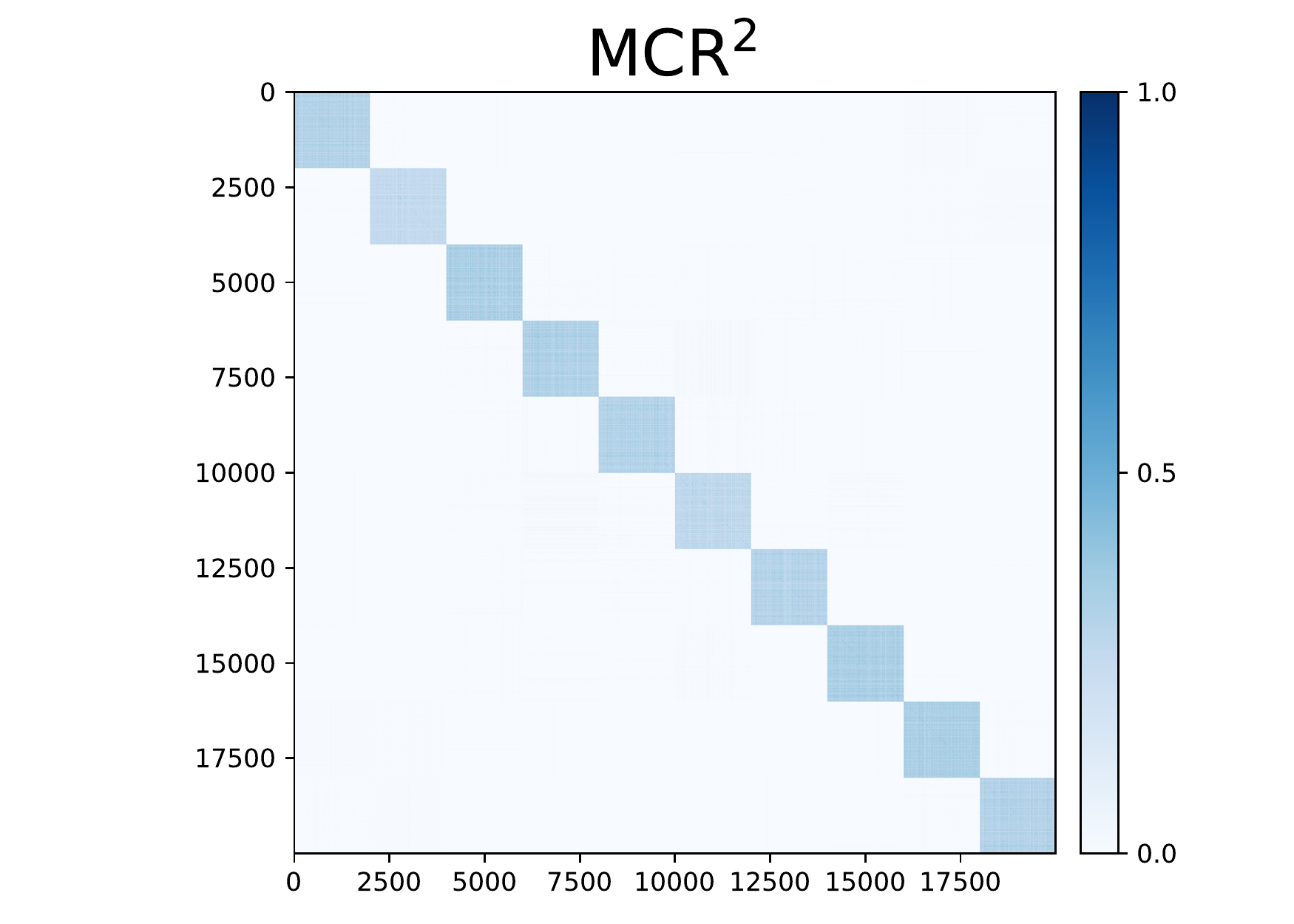}
      \includegraphics[width=0.45\linewidth]{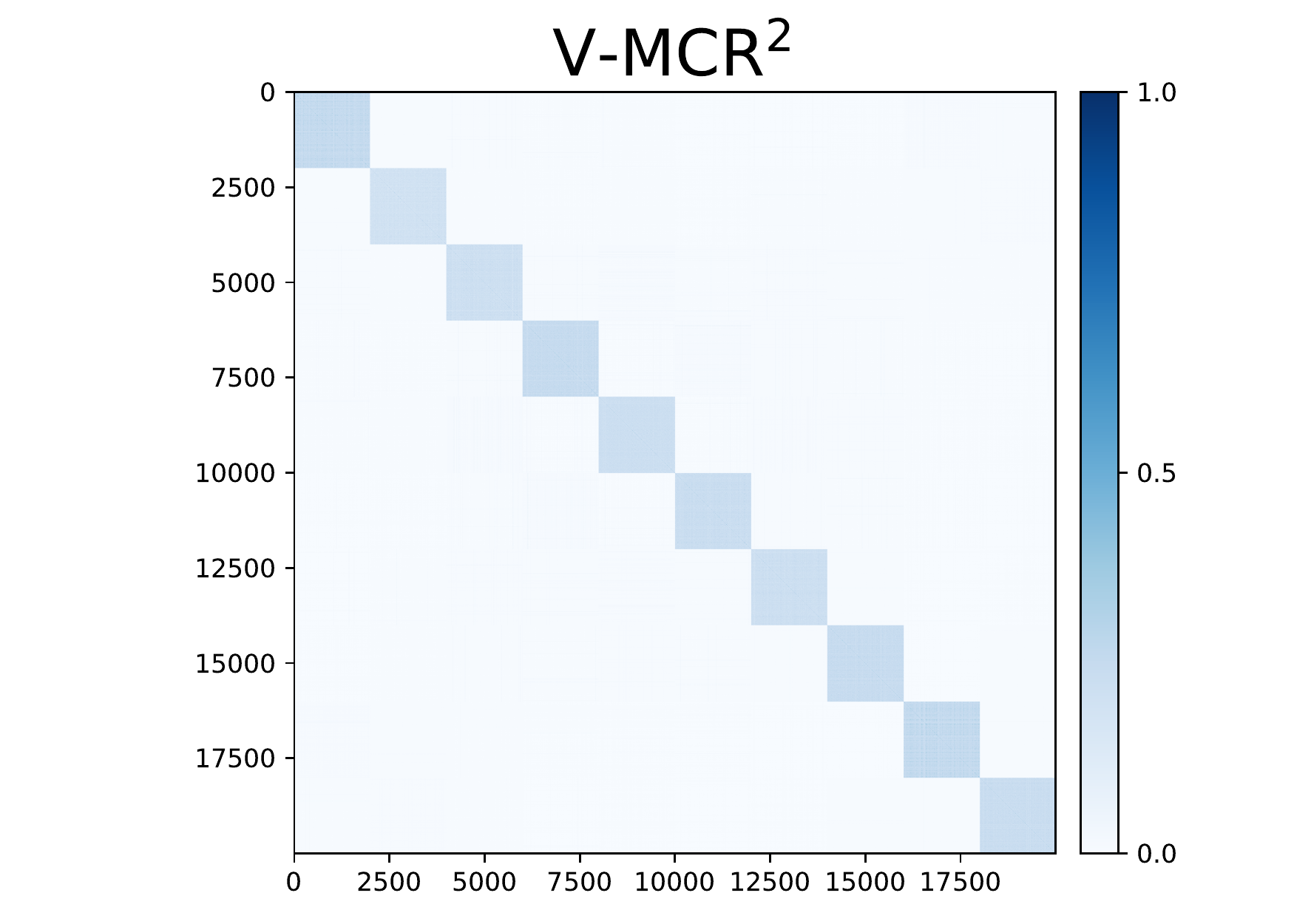}
      \caption{CIFAR-10}
      \label{fig:cifar10_heatmap}
    \end{subfigure} 
    \begin{subfigure}{.45\textwidth}
      \centering
      \includegraphics[width=0.45\linewidth]{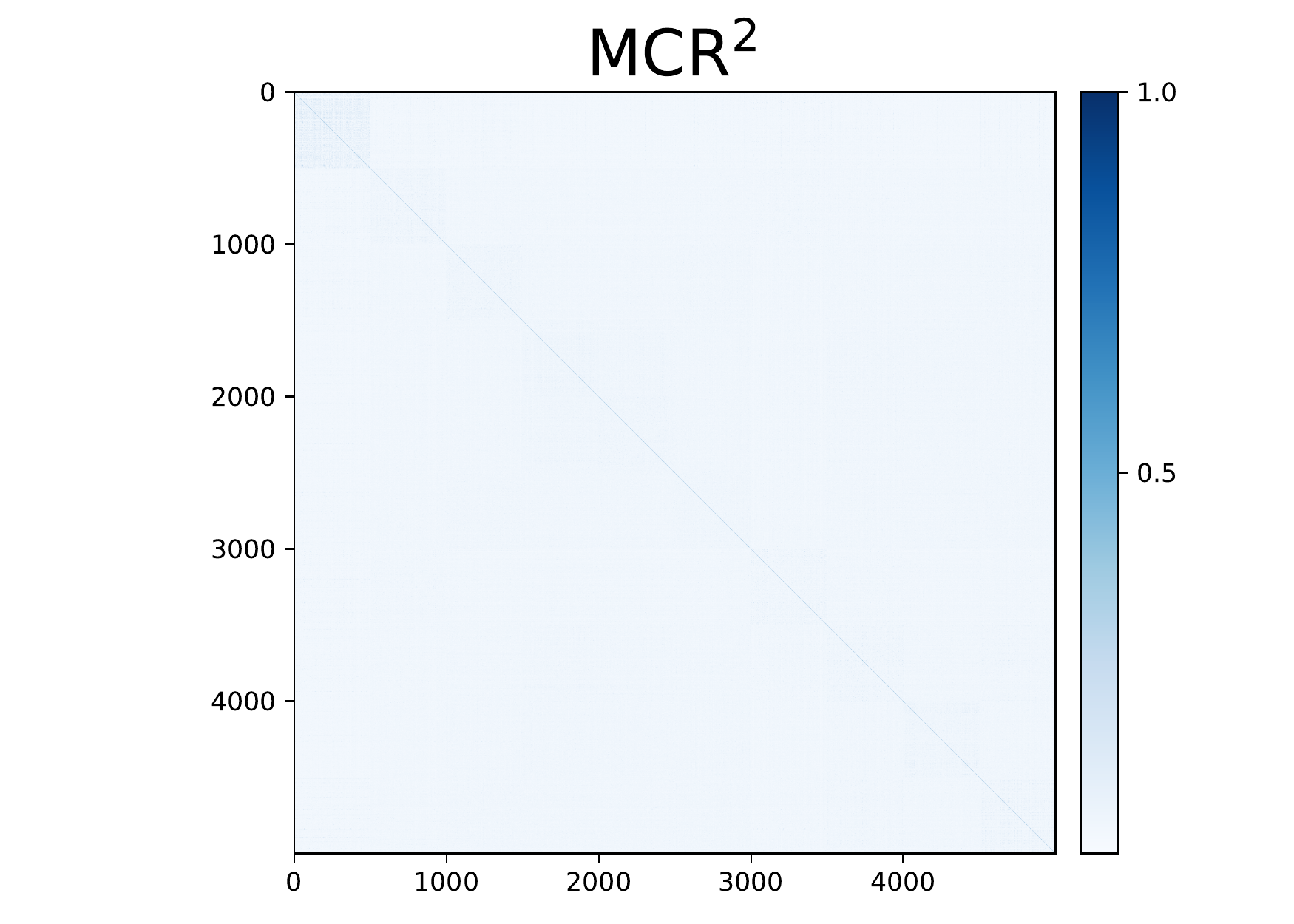}
      \includegraphics[width=0.45\linewidth]{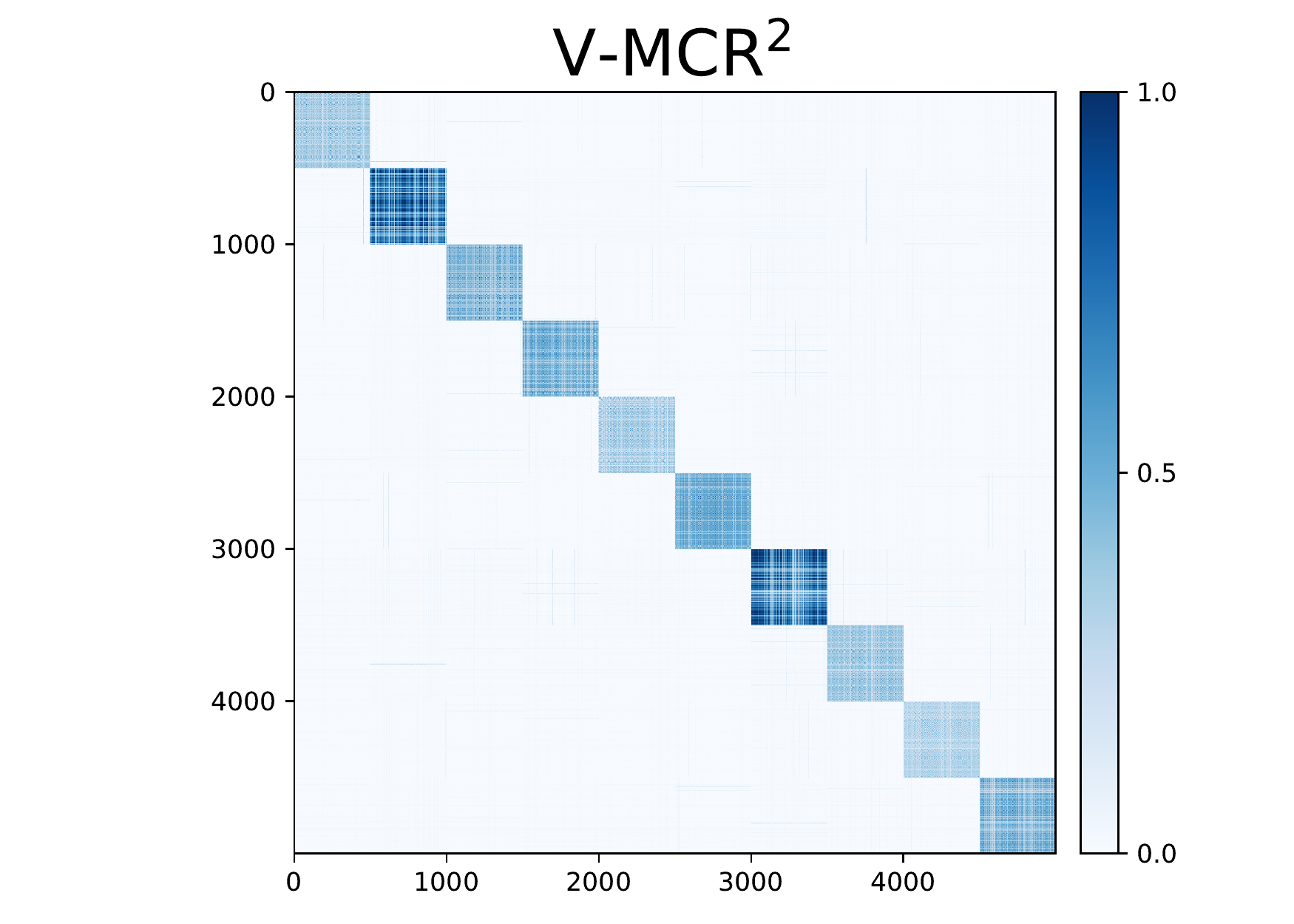}
      \caption{CIFAR-100}
      \label{fig:cifar100_heatmap} 
    \end{subfigure}
    \begin{subfigure}{.45\textwidth}
      \centering
      \includegraphics[width=0.45\linewidth]{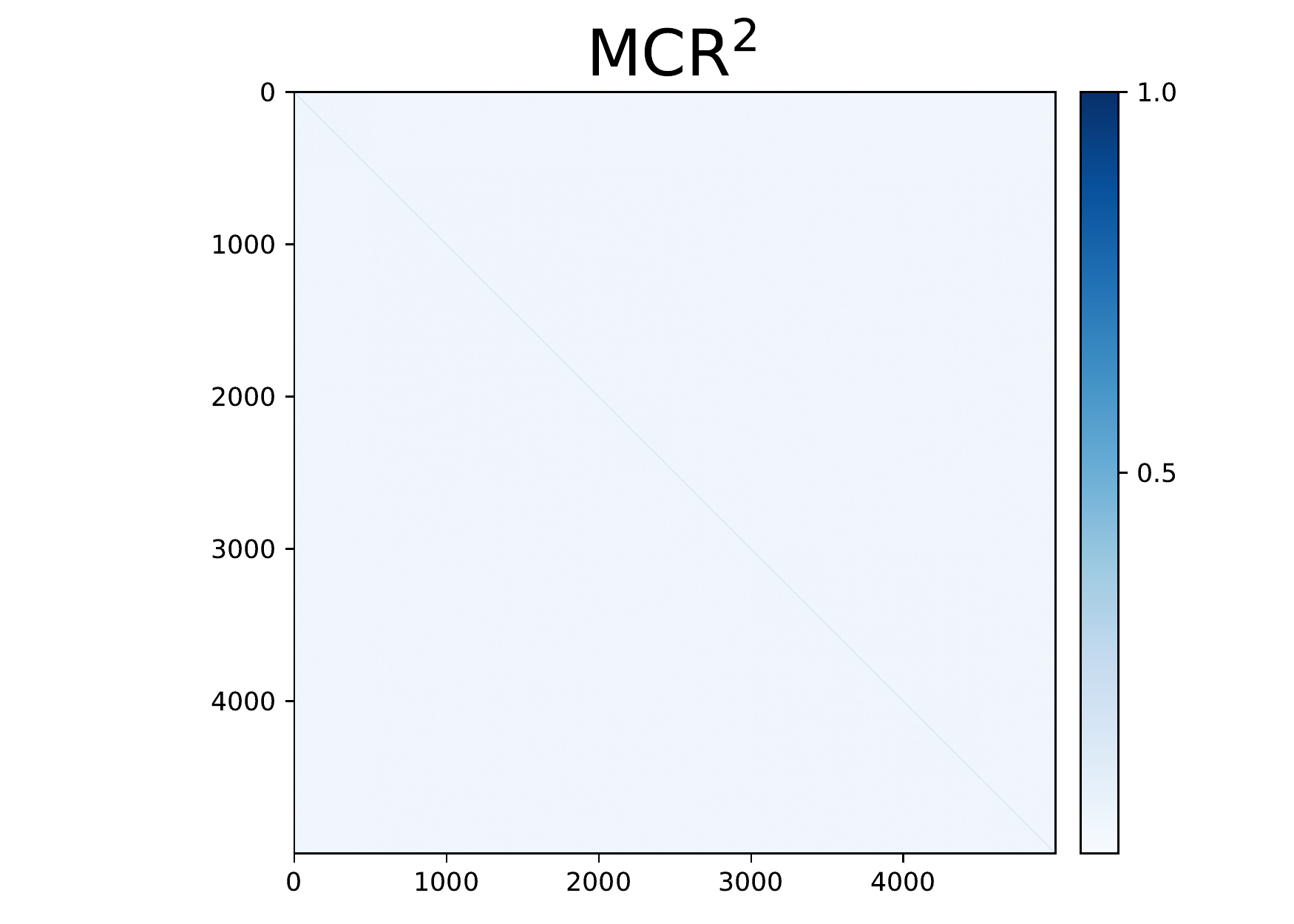}
      \includegraphics[width=0.45\linewidth]{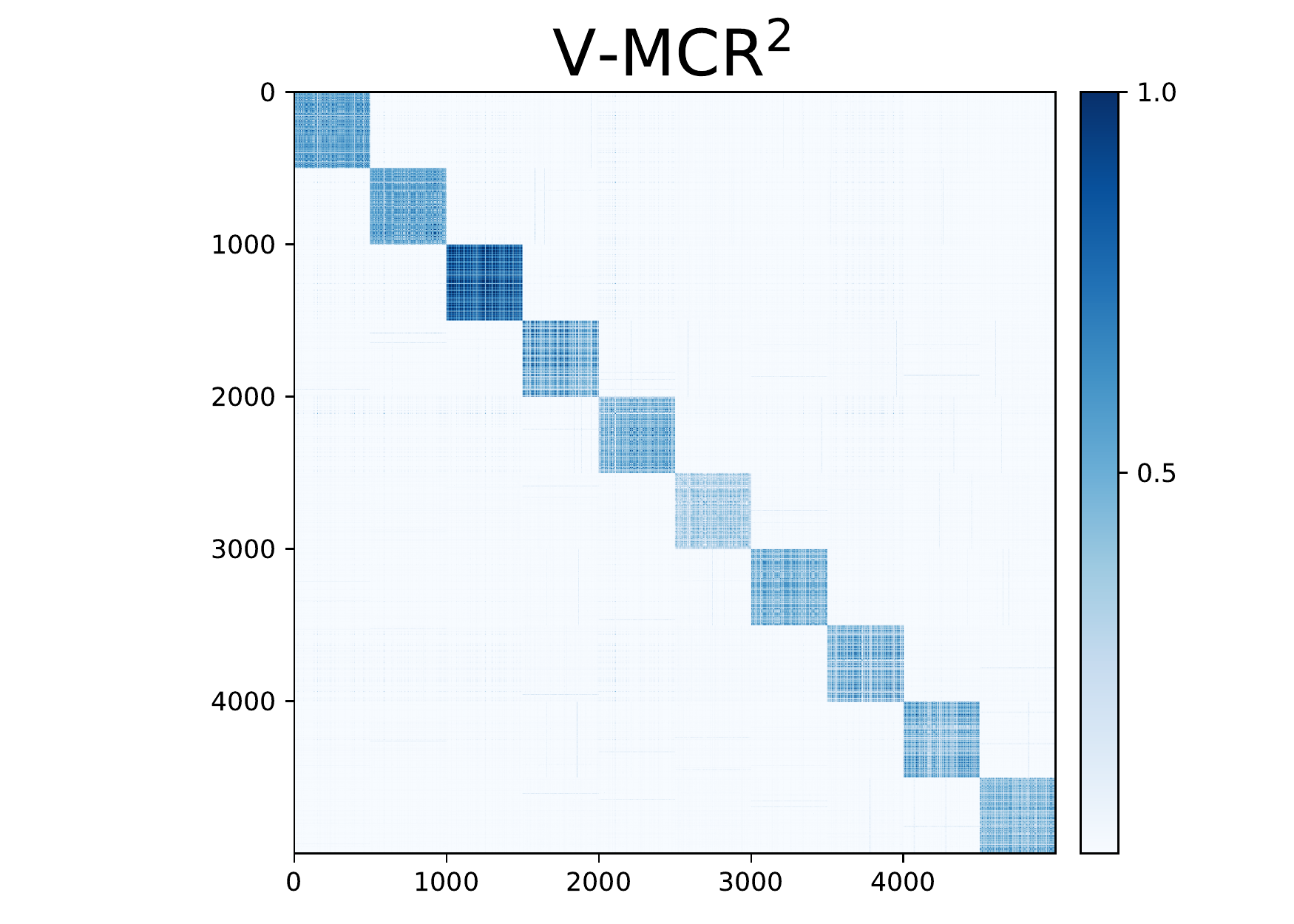}
      \caption{Tiny ImageNet}
      \label{fig:imagenet_heatmap}
    \end{subfigure}
    
    \caption{\textbf{Inner product of representations.} We plot the heatmap of $|\mb Z_\theta^\top \mb Z_\theta|$ where $\mb Z_\theta$ are the representations of the training data, ordered by the class they belong to. For CIFAR-100 and Tiny ImageNet, 10 classes are randomly chosen. If the classes lie on low-dimensional orthogonal subspaces, we expect to see a block diagonal structure.  }
    \label{fig:heatmap} 
    
    \vspace{-5mm}
\end{figure}

\subsection{V-MCR$^2$ Obtains Better Representations}
\label{sec:representations}
There are two properties of the representation that we aim to attain by optimizing the MCR$^2$ objective. We say that a representation of the training data is of `high quality' if points from different classes lie on separate, orthogonal subspaces, and the union of these subspaces span as many dimensions as possible. In particular, the orthogonal property is important in order to classify the points using the nearest subspace algorithm in Section \ref{sec:nearsub}. To check the orthogonality of subspaces learned by $f_\theta$, we report the inner product between every pair of training points as a heatmap in Figure \ref{fig:heatmap}. Namely, we sort the columns of $\mb Z_\theta$ by class 
and compute $|\mb Z_\theta^\top \mb Z_\theta|$. Ideally, we want $|\mb Z_\theta ^\top \mb Z_\theta|$ to have a block diagonal structure, with $(\mb Z_\theta)_i^\top (\mb Z_\theta)_j \approx 0$ for $i, j$ notating points from different classes. 

Figure \ref{fig:heatmap} shows heatmaps for MNIST, CIFAR-10, and CIFAR-100 after 2000 training epochs, when $\Delta R$ has converged for both MCR$^2$ and V-MCR$^2$ (Figure \ref{fig:convergence}). For MNIST and CIFAR-10, note that both MCR$^2$ and V-MCR$^2$ obtain a block diagonal structure. However, on CIFAR-100 and Tiny Imagenet, we observe no block diagonal structure after MCR$^2$ training, whereas we see a clear block diagonal structure for V-MCR$^2$. These findings suggest that V-\mcr2 training is more robust to avoiding poor local minima than training on the original MCR$^2$ model, particularly as the number of classes increase. 
We leave a rigorous study of these phenomena for a future work.

\subsection{Performance on Classification Tasks}
In Table \ref{tab:performance}, we present test accuracies on the four aforementioned datasets when trained under original MCR$^2$ and V-MCR$^2$ objectives. In addition, we also train a separate model for each dataset by using CE loss as a reference classifier and report its test accuracy. Notice that the goal of our study here is {\em not} about achieving the best possible classification accuracy on these datasets -- the training procedure and architectures used here are not optimal for that purpose\footnote{We use simple training practices (\eg input data downsampling, minimal data augmentation and training from scratch), leading to about 15\% and 30\% difference in performance on CIFAR-100 and Tiny ImageNet as reported in \cite{yun2020regularizing} from much more carefully engineered training recipes.
}. Instead, we make fair comparison of all methods on the same networks and datasets to justify the computational efficiency and effectiveness of the proposed method. To ensure fairness, we initialize this reference model with the same architecture along with other hyperparameters\footnote{For experiments using CE, we use a higher learning rate of $10^{-2}$ to improve convergence for CE training.} as in V-MCR$^2$ experiments and attach a final linear classifier with output dimension corresponding to the number of classes.

As shown in Table \ref{tab:performance}, when trained on datasets with a small number of classes, all three training objectives can reach competitive classification performance. 
We can observe that both MCR$^2$ and V-MCR$^2$ objectives are comparable to CE in these small-scale datasets.  Again on CIFAR-100 we observe that the poor local minima obtained by training the original \mcr2 objective results in a poor test accuracy, while the V-MCR$^2$ model achieves comparable performance to training the same network with CE.
%
%
Due to limited resources and heavy computational requirements of MCR$^2$ training with a large number of classes, we can only report the result on Tiny ImageNet by optimizing MCR$^2$ after 200 epochs while the other two objectives complete the full training session of 2000 epochs. 


\begin{table}[!htb]
\begin{small}
    \centering
    \begin{tabular}{|c|c|c|c|}
    \hline
         \textbf{Dataset} & \textbf{Objective} & \textbf{Training} $\mb{\Delta R}$  & \textbf{Test Accuracy} \\
         \hline 
          & MCR$^2$ & 44.6429 & 0.9785 \\ 
         MNIST & V-MCR$^2$ & 44.2117 &  \textbf{0.9788}\\ 
          & CE & - &  0.9738\\ 
          \hline 
          & MCR$^2$ & 49.40 & 0.8956\\ 
         CIFAR-10 & V-MCR$^2$ & 48.43 &  \textbf{0.8997}\\ 
          & CE & - &  0.8665\\ 
          \hline 
          & MCR$^2$ & 226.0519 &  0.2421 \\ 
         CIFAR-100 & V-MCR$^2$ & 218.0185 &  \textbf{0.5872}\\ 
          & CE & - & 0.5840\\ 
          \hline 
         Tiny & MCR$^2$ & 227.6468 &  0.1319\\ 
         ImageNet & V-MCR$^2$ & 231.1538&  \textbf{0.2665}\\ 
         200 & CE & - &  0.1907\\ 
          \hline
    \end{tabular}
    \caption{\textbf{Comparison of classification performance.} We evaluate the training $\Delta R$ and test accuracy of CE, MCR$^2$, and V-MCR$^2$ after 2000 training epochs for MNIST, CIFAR-10, and CIFAR-100. For Tiny ImageNet, we report the results after 200 epochs for MCR$^2$ due to very slow training, though note that $\Delta R$ had already converged (Fig. \ref{fig:imagenet_loss}). For V-MCR$^2$ and CE, results are after 2000 epochs. For CIFAR-100 and Tiny ImageNet, we report results for $d = 500$. See Appendix for results for other choices of $d$. }
    \label{tab:performance}
\end{small}
\end{table}

\section{Conclusion}
Building on Yu et al. \cite{mcr}, we propose an alternative cost-efficient formulation of the MCR$^2$ objective that is scalable to datasets with a large number of classes. Namely, for CIFAR-100 and Tiny ImageNet, we observed a $5 \times$ and $12\times$ speedup per training epoch, respectively. The gain would be even more significant as the number of classes increases. Additionally, we show that we do not make any compromises when it comes to performance by using this approximate, variational formulation. In fact, in all datasets we tested, not only does V-MCR$^2$  reach similar $\Delta R$ values as MCR$^2$, the learned representations of V-MCR$^2$ were just as good, oftentimes better, than those learned by MCR$^2$ for the same number of training epochs. On a related note, we observe an interesting phenomenon that directly optimizing the MCR$^2$ objective is not only slow, but often completely fails to learn the desired orthogonal subspace structures when the number of classes increases even after $\Delta R$ has successfully converged. In our experiments, we found that V-MCR$^2$ is surprisingly not prone to this problem, suggesting that V-MCR$^2$ allows for better control over learning the desired representations. By these means, we find that V-MCR$^2$ is a very promising adaptation of MCR$^2$ and we leave a rigorous comparison of the representations learned by MCR$^2$ and V-MCR$^2$ for future work.

\medskip \noindent
\textbf{Acknowledgements}
This work was partially supported by NSF Fellowship DGE2139757, NSF grants 1704458 and 2031985, and the Northrop Grumman
Mission Systems Research in Applications for Learning Machines (REALM) initiative. Yi Ma was funded by ONR grants N00014-20-1-2002 and N00014-22-1-2102, the joint Simons Foundation-NSF DMS grant $\#2031899$.

\newpage 
{
\bibliographystyle{ieee_fullname}
\bibliography{egbib}
}
\onecolumn

\section{Appendix}
\subsection{Derivation of the Lipschitz Constants}
To ensure stability of the gradient-based method, we scale the learning rate of $\Gamma$ and $\mb A$ by (the inverse of) approximate upper-bounds of the Lipschitz constants of the gradients, $\frac{1}{L_\Gamma}$ and $\frac{1}{L_{\mb A}}$, respectively. To simplify calculations, we only bound the Lipschitz constant of the gradients with respect to the matrix approximation term $M$, as we note this term typically dominates the Lipschitz constant of the gradient. Additionally, for notational simplicity we show the derivation for balanced classes which implies $\frac{1}{m \gamma_j} = \frac{k}{m}$.

Now, consider the function $f(\Gamma,\mb A)=\frac{\mu k }{2m}\sum_{j=1}^k\norm{\mathbf{Z} \Diag(\mb{\Pi})_j \mathbf{Z}^\top - \Gamma \Diag(\mb A_j) \Gamma^\top}{F}^2$, and note that with some simple algebra one can show the following equivalence:

\begin{equation}
\begin{split}
f(\Gamma,\mb A) =& \frac{\mu k}{2 m} \sum_{j=1}^k \| \mb Z \Diag(\mb \Pi_j) \mb Z^\top - \Gamma \Diag(\mb A_j) \Gamma^\top \|_F^2  \\
=&\frac{\mu k}{2 m} \langle \mb \Pi \mb \Pi^\top, (\mb Z^\top \mb Z)^{\odot 2} \rangle - \frac{\mu k}{m}\langle \mb A \mb \Pi^\top, (\Gamma^\top \mb Z)^{\odot 2} \rangle + \frac{\mu k}{2 m} \langle \mb A \mb A^\top, (\Gamma^\top \Gamma)^{\odot 2} \rangle 
\end{split}
\end{equation}
where $(\cdot)^{\odot 2}$ denotes raising to the second power entry-wise.

From this, the relevant gradients are
\begin{align}
    \nabla_\Gamma f & = -\frac{2\mu k}{m} \sum_{j=1}^k \left( \underbrace{\mathbf{Z} \Diag(\mb{\Pi}_j) \mathbf{Z}^\top \Gamma \Diag(\mb A_j)}_{:=g_j(\Gamma)} -  \underbrace{\Gamma \Diag(\mb A_j) \Gamma^\top \Gamma \Diag(\mb A_j)}_{:=h_j(\Gamma)} \right)\\
    \nabla_{\mb A} f &= \frac{\mu k}{m} \left[ (\Gamma^\top \Gamma)^{\odot 2} \mb A - (\Gamma^\top \mb Z)^{\odot 2} \mb \Pi \right] 
\end{align}
We first bound the Lipschitz constant of $\nabla_\Gamma f$. For $g_j$, we have 
\begin{align}
    \frac{\|g_j(\Gamma+ \Delta \Gamma) - g_j(\Gamma)\|_F}{\|\Delta \Gamma\|_F} &= \frac{\|\mathbf{Z} \Diag(\mb{\Pi}_j) \mathbf{Z}^\top \Delta\Gamma \Diag(\mb A_j)\|_F}{\|\Delta \Gamma \|_F} \\
    & \leq \|\mathbf{Z} \Diag(\mb{\Pi}_j) \mathbf{Z}^\top \|_F  \|\mb A_j\|_{\infty},
\end{align}
where the inequality is due to the sub-multiplicative property of the Frobenius norm and the operator norm inequality. Now we consider $h_j$. Here we make the simplifying assumption that in addition to each column of $\Gamma$ being a unit vector, $\Gamma$ also approximately satisfies $\Gamma^\top \Gamma = \mb I$, which is known to be true near the globally optimal solution from the analysis of \cite{mcr}. Now, we have the following approximation 
\begin{equation}
    \Bar{h}_j(\Gamma) \approx \Gamma \Diag(\mb A_j)^2.
\end{equation} which implies,
\begin{align}
    \frac{\|\Bar{h}_j(\Gamma+ \Delta \Gamma) - \Bar{h}_j(\Gamma)\|_F}{\|\Delta \Gamma\|_F} &\approx \frac{\|\Delta\Gamma \Diag(\mb A_j)^2\|_F}{\|\Delta \Gamma \|_F}  \\ & \leq 
    \|\mb A_j\|_\infty^2
\end{align}
Since $f(\Gamma, \mb A)$ has a summation over class $j$, we get the following approximate upper bound for the Lipschitz constant
\begin{equation}
    L_\Gamma = \frac{2\mu k}{m} \sum_{j=1}^k \|\mathbf{Z} \Diag(\mb{\Pi}_j) \mathbf{Z}^\top \|_F  \|\mb A_j\|_{\infty} + \|\mb A_j\|_\infty^2
\end{equation} 

For the upper bound of the Lipschitz constant of $\nabla_{\mb A} f$ we simply have 
\begin{align}
     \frac{\left \|\nabla_{\mb A} f \big{|}_{\mb A + \Delta \mb A} - \nabla_{\mb A} f\big{|}_{\mb A}\right \|_F}{\|\Delta \mb A\|_F} &=
     \frac{\| \frac{\mu k}{m} (\Gamma^\top \Gamma)^{\odot 2} \Delta \mb A \|_F} { \| \Delta \mb A\|_F} \\
     &\leq \frac{\mu k}{m}\|(\Gamma^\top \Gamma)^{\odot 2}\|_F = L_A
\end{align}

\subsection{Architecture}
We utilize the following architectures for the experiments in Section 5. We use a fairly simple architecture for MNIST, and for the other datasets, we use a slightly modified version of ResNet18. Note \texttt{d} is the feature dimension. 

\setcounter{algorithm}{0}
\floatname{algorithm}{Architecture}
\begin{algorithm*}
\caption{Neural Network Architecture for MNIST}\label{alg:mnist_arch}
\begin{algorithmic}[1]
\State \texttt{Conv2d(in\_channel=1, out\_channel=32, kernel=3, stride=1)}
\State \texttt{ReLU()}
\State \texttt{Conv2d(in\_channel=32, out\_channel=64, kernel=3, stride=1)}
\State \texttt{ReLU()}
\State \texttt{MaxPool2d(kernel=2, stride=None)}
\State \texttt{Dropout(p=0.25)}
\State \texttt{Flatten()}
\State \texttt{Linear(12544, d)}
\State \texttt{ReLU()}
\State \texttt{Dropout(p=0.5)}
\State \texttt{Linear(d, d)}
\State \texttt{Normalize()}
\end{algorithmic}
\end{algorithm*}



For CIFAR-10/100 and Tiny Imagenet, we use the Torchvision ResNet18 model as the featurizer, but we remove the final layer of the ResNet18 and replace it with the following to reshape the output into the desired feature dimension $d$. We also normalize the output at the end to fulfill the constraint that the features lie on the unit sphere in the MCR$^2$ objective.

\floatname{algorithm}{Architecture}
\begin{algorithm*}
\caption{Reshaping Layers for ResNet18}\label{alg:resnet_arch}
\begin{algorithmic}[1]
\State  \texttt{Linear(512, 512, bias=False)}
\State  \texttt{BatchNorm1d()}
\State \texttt{ReLU()}
\State \texttt{Linear(512, d, bias=True)}
\State \texttt{Normalize()}
\end{algorithmic}
\end{algorithm*}

For cross-entropy experiments, we add another linear layer on top to map the output of the featurizer to logits.

\subsection{Data Augmentation}
We utilize the following data augmentations for the experiments in Section 5. 


\setcounter{algorithm}{0}
\floatname{algorithm}{Transformations}
\begin{algorithm*}
\caption{Transformations for MNIST, CIFAR-10, CIFAR-100}\label{alg:transform1}
\begin{algorithmic}[1]
\State \texttt{import torchvision.transforms as transforms}
\State \texttt{TRANSFORM = transforms.Compose([}
\State \quad \quad \quad \quad  \texttt{transforms.RandomCrop(32, padding=8),}
\State \quad \quad \quad \quad \texttt{transforms.RandomHorizontalFlip(),}
\State \quad \quad \quad \quad \texttt{transforms.ToTensor()])}
\end{algorithmic}
\end{algorithm*}

\floatname{algorithm}{Transformations}
\begin{algorithm*}
\caption{Transformations for Tiny ImageNet}
\begin{algorithmic}[1]
\State \texttt{import torchvision.transforms as transforms}
\State \texttt{TRANSFORM = transforms.Compose([}
\State \quad \quad \quad \quad\texttt{transforms.Resize(32)}
\State \quad \quad \quad \quad\texttt{transforms.RandomHorizontalFlip()}
\State \quad \quad \quad \quad\texttt{transforms.ToTensor()])}
\end{algorithmic}
\end{algorithm*}

\subsection{Additional Experiments}
The experiments conducted and presented in the main body, Table 2,  were only meant to compare the computational efficiency for different methods under the same conditions. The settings however were not chosen to optimize the classification performance since we did not conduct data augmentation or other training recipes normally adopted, e.g. see \cite{yun2020regularizing}. 

Here we report experimental results on CIFAR-100 and Tiny ImageNet by training with a similar training recipe adopted in
\cite{yun2020regularizing}. Specifically, as in \cite{yun2020regularizing}, we use PreAct ResNet-18 as the backbone and train with both cross-entropy loss and V-MCR$^2$. All networks are trained by SGD with momentum 0.9, and weight decay of $10^{-4}$. Similar to \cite{yun2020regularizing}, we set the learning rate to be $10^{-1}$ for the first half of training epochs, divide it by $10$ for the next quarter of epochs, and finally divide it again by $10$ for the remaining iterations. Additionally, we utilize the same transformations to augment the data.
Note that the only difference we make between our strategy and the one in \cite{yun2020regularizing} is the choice of batch size and total number of training epochs (both specified in Section 4.3) to ensure fair comparison between CE and V-MCR$^2$. In addition to the training strategy, we set $\mu = 10^{-1}$ for V-MCR$^2$ as we find a smaller value for $\mu$ helps stabilize training in the early stage. We keep all other hyperparameters identical to ones specified in Section 4.3.

\begin{table}[!htb]
\begin{small}
    \centering
    \begin{tabular}{|c|c|c|c|}
    \hline
         \textbf{Dataset} & \textbf{Objective} & \textbf{Training} $\mb{\Delta R}$  & \textbf{Test Accuracy} \\
         \hline
         CIFAR-100 & V-MCR$^2$ & 130.2177 &  0.6951\\ 
                  & CE & - & 0.7146\\ 
          \hline 
         Tiny ImageNet & V-MCR$^2$ & 134.9504&  0.4189\\ 
         200 & CE & - &  0.4843\\ 
          \hline
    \end{tabular}
    \caption{\textbf{Comparison of classification performance.} We evaluate the training $\Delta R$ and test accuracy of V-MCR$^2$ after 2000 training epochs for CIFAR-100 and Tiny ImageNet. For CE, we evaluate the test accuracy at each epoch and report the highest test accuracy achieved across the 2000 epochs.  \cite{yun2020regularizing} similarly achieved an accuracy 0.7529 for CIFAR-100 and 0.5647 for Tiny ImageNet for CE training. With V-MCR$^2$, we achieve a lower test accuracy. Note that the hyperparameters for V-MCR$^2$ such as latch frequency, dictionary size, and $\mu$ were not rigorously tuned. Additionally, due to the computational cost of the nearest subspace classification algorithm, we only evaluate the test performance at the 2000 epoch mark.}
    \label{tab:additional_performance}
\end{small}    
\end{table}
\vspace{-5mm}

\end{document}